\documentclass{article} 
\usepackage{iclr2023_conference,times}

\usepackage{amsmath,amsfonts,bm}

\def\eqref#1{equation~\ref{#1}}

\def\1{\bm{1}}

\def\vh{{\bm{h}}}

\def\vu{{\bm{u}}}

\def\mB{{\bm{B}}}

\def\mE{{\bm{E}}}

\def\mG{{\bm{G}}}

\def\mJ{{\bm{J}}}

\def\mU{{\bm{U}}}
\def\mV{{\bm{V}}}

\def\mX{{\bm{X}}}
\def\mY{{\bm{Y}}}

\DeclareMathAlphabet{\mathsfit}{\encodingdefault}{\sfdefault}{m}{sl}
\SetMathAlphabet{\mathsfit}{bold}{\encodingdefault}{\sfdefault}{bx}{n}

\def\Pa{{\text{\textbf{Pa}}}}

\newcommand{\E}{\mathbb{E}}

\newcommand{\R}{\mathbb{R}}

\newcommand{\KL}{D_{\mathrm{KL}}}

\usepackage{amsthm, amssymb}
\usepackage{hyperref}
\usepackage[nameinlink,capitalize]{cleveref}
\usepackage{booktabs}
\usepackage{comment}
\usepackage{graphicx}
\newif\ifshowcomments
\showcommentsfalse

\theoremstyle{definition}
\newtheorem{definition}{Definition}[section]
\newtheorem{theorem}{Theorem}
\newtheorem{corollary}{Corollary}[theorem]

\theoremstyle{definition}
\newtheorem{assumption}{Assumption}

\newcommand{\ModelName}{Rhino}

\newcommand{\multitimeX}[1]{\mX_{0:T}^{(#1)}}
\newcommand{\singletimeX}[1][0]{\mX_{#1:T}}
\newcommand{\singletimeY}[1][0]{\mY_{#1:T}}
\newcommand{\singletimeXsource}[1][0]{\mX_{#1:\sourcelength}}
\newcommand{\singletimeYsource}[1][0]{\mY_{#1:\sourcelength}}

\newcommand{\singletimeXfull}[2]{\mX_{#1:#2}}
\newcommand{\singletimeYfull}[2]{\mY_{#1:#2}}

\newcommand{\Graph}[1][0]{\mG_{#1:K}}
\newcommand{\singlegraph}{\mG_{0}}
\DeclareMathOperator{\tr}{tr}
\DeclareMathOperator{\CATE}{CATE}
\DeclareMathOperator{\doop}{do}
\def\vardist{{q_\phi(\mG)}}
\def\sourcelength{{\mathcal{S}}}
\def\sourcenodes{{\mX_{0:\sourcelength}}}

\newcommand{\PaGst}[1][i]{\Pa_{G}^{#1}(<t)}
\newcommand{\PaGt}[1][i]{\Pa_{G}^{#1}(t)}

 \newcommand{\ind}{\perp\!\!\!\!\perp} 

\newcommand{\PaGstXY}[1][X,Y]{\Pa_{G}^{#1}(<t)}

\newcommand{\bPaGstXY}[1][X,Y]{\overline{\Pa}_{G'}^{#1}(<t)}

\def\historyGX{{\vh_G^X}}
\def\historyGY{{\vh_G^Y}}

\def\bhistoryGX{{\bar{\vh}_{G'}^X}}
\def\bhistoryGY{{\bar{\vh}_{G'}^Y}}

\def\bnu{{\bar{\nu}}}

\def\alphat{{\alpha_t}}
\def\betat{{\beta_t}}

\newenvironment{sproof}{%
  \proof}{\endproof}

\usepackage{wrapfig}

\usepackage{hyperref}
\usepackage{url}

\iclrfinalcopy

\title{Rhino: Deep Causal Temporal Relationship Learning with history-dependent noise}

\author{Wenbo Gong, Joel Jennings, Cheng Zhang \& Nick Pawlowski\\
Microsoft Research\\
Cambridge, UK\\
\texttt{\{t-gongwenbo, joeljennings, cheng.zhang, nick.pawlowski\}}\\
\texttt{@microsoft.com}
}

\begin{document}

\maketitle

\begin{abstract}
Discovering causal relationships between different variables from time series data has been a long-standing challenge for many domains {such as climate science, finance and healthcare.}
{Given the the complexity of real-world relationships and the nature of observations in discrete time,} {causal discovery methods} need to consider non-linear relations between variables, instantaneous effects and history dependent noise (the change of noise distribution due to past actions).
However, previous works do not offer a solution addressing all these problems together. 
In this paper, we propose a 
{novel causal relationship learning framework for time-series data}, called \ModelName{}, which combines vector auto-regression, deep learning and variational inference to model non-linear relationships with instantaneous effects while allowing the noise distribution to be modulated by historical observations. Theoretically, we prove the structural identifiability of \ModelName{}. Our empirical results from extensive synthetic experiments and two real-world benchmarks demonstrate better discovery performance compared to relevant baselines, with ablation studies revealing its robustness under model misspecification.
\end{abstract}


\section{Introduction}
\label{sec: introduction}
Time series data is a collection of data points recorded at different timestamps describing a pattern of chronological change. {Identifying the causal relations between different variables and their interactions through time \citep{spirtes2000causation, berzuini2012causality, guo2020survey, peters2017elements} is essential for many applications e.g.~climate science, health care, etc. } 
Randomized control trials are the gold standard for discovering such relationships, but may be unavailable due to cost and ethical constraints. Therefore, causal discovery with just observational data is important and fundamental to many real-world applications \citep{lowe2022amortized, bussmann2021neural,moraffah2021causal,wu2020discovering, runge2018causal,tank2018neural, hyvarinen2010estimation, pamfil2020dynotears}.

The task of temporal causal discovery can be challenging for several reasons: (1) relations between variables can be non-linear in the real world; (2) with a slow sampling interval, everything happens in between will be aggregated into the same timestamp, i.e.~instantaneous effect; (3) the noise may be non-stationary (its distribution depends on the past observations), i.e.~history-dependent noise. For example, in stock markets, the announcements of some decisions from a leading company after the market closes may have complex effects (i.e.~non-linearity) on its stock price immediately after the market opening (i.e.~slow sampling interval and instantaneous effect) and its price volatility may also be changed (i.e.~history-dependent noise). 
Similarly, in education, students that recently earned good marks on algebra tests should also score well on an upcoming algebra exam with little variation (i.e.~history-dependent noise).

To the best of our knowledge, 
{existing frameworks' performances suffer in many real-world scenarios as they cannot address these aspects in a satisfactory way.} Especially, history-dependent noise has been rarely considered in past. A large category of the preceding works, called \emph{Granger causality} \citep{granger1969investigating}, is based on the fact that cause-effect relationships can never go against time. Despite many recent advances \citep{wu2020discovering, shojaie2010discovering, siggiridou2015granger, amornbunchornvej2019variable, lowe2022amortized, tank2018neural, bussmann2021neural, dang2018seq2graph, xu2019scalable}, they all rely on the absence of instantaneous effects with a fixed noise distribution. Constraint-based methods have also been extended for time series causal discovery \citep{runge2018causal, runge2020discovering}, {which is commonly applied by folding the time-series. This introduced new assumptions and translated the aforementioned requirements to challenges in conditional independence testing \citep{shah2020hardness}.}
Additionally, they require a stronger faithfulness assumption and can only identify the causal graph up to a Markov equivalence class without detailed functional relationships.

An alternative line of research leverages the development of causal discovery with functional causal models \citep{hyvarinen2010estimation, pamfil2020dynotears, peters2013causal}. They can model both instantaneous and lagged effects as long as they have theoretically guaranteed \emph{structural identifiability}. Unfortunately, they do not consider history-dependent noise. One central challenge of modelling this dependency is that noise depending on the lagged parents may break the model structural identifiability. For static data, \citet{khemakhem2021causal} proves the structural identifiability only when this dependency is restricted to a simple functional form. Thus, the key research question is whether the identifiability can be preserved with complex historical dependencies in the temporal setting.

Motivated by these requirements, we propose a novel temporal discovery called {\ModelName{}} (\emph{deep causal temporal \underline{r}elationship learning with \underline{hi}story dependent \underline{no}ise}), which can model non-linear lagged and instantaneous effects with flexible history-dependent noise. Our contributions are:
\begin{itemize}
    \item A novel {causal discovery framework }
    called {\ModelName{}}, which combines vector auto-regression and deep learning to model non-linear lagged and instantaneous effects with history-dependent noise. We also propose a principled training framework using variational inference.
    \item We prove that {\ModelName{}} is structurally identifiable
    . To achieve this, we provide general conditions for structural identifiability with history-dependent noise, of which \ModelName{} is a special case. Furthermore, we clarify relations to several previous works.
    \item We conduct extensive synthetic experiments with ablation studies to demonstrate the advantages of \ModelName{} and its robustness under model misspecification. Additionally, we compare its performance to a wide range of baselines in two real-world discovery benchmarks.
\end{itemize}


\section{Background}
\label{sec: Background}
In this section, we briefly introduce necessary preliminaries for \ModelName{}. In particular, we focus on structural equation models, Granger causality \citep{granger1969investigating} and vector auto-regression. 

\paragraph{Structural Equation Models (SEMs)} 
Consider $\mX\in\mathbb{R}^D$ with $D$ variables, SEM describes the causal relationships between them given a causal graph $\mG$:
\begin{equation}
    X^i = f_i(\Pa_G^i, \epsilon^i)
\end{equation}
where $\Pa_G^i$ are the parents of node $i$ and $\epsilon^i$ are mutually independent noise variables.
Under the context of multivariate time series, 
$\mX_t=\left(X_t^i\right)_{i\in\mV}$ where $\mV$ is a set of nodes with size $D$, the corresponding SEM given a temporal causal graph $\mG$ is
\begin{equation}
    X_t^i = f_{i,t}(\PaGst, \PaGt, \epsilon_{t}^i),
    \label{eq: Temporal SEM}
\end{equation}
where $\PaGst$ contains the parent values specified by $G$ in previous time (\emph{lagged parents}); $\PaGt$ are the parents at the current time $t$ (\emph{instantaneous parents}). 
The above SEM induces a joint distribution over the stationary time series $\{\mX_{t}\}_{t=0}^T$ (see \cref{assumption 1} in \cref{app: Structure Identifiability} for the definition). 
However, functional causal models with the above general form cannot be directly used for causal discovery due to the structural unidentifiability (Lemma 1, \citet{zhang2015estimation} 
One way to solve this is sacrificing the flexibility by restricting the functional class. For example, additive noise models (ANM), \citep{hoyer2008nonlinear}
\begin{equation}
    X^i=f_i(\Pa_G(X^i)) + \epsilon_i,
    \label{eq: DECI SEM}
\end{equation}
which have recently been used for causal reasoning with non-temporal data \citep{geffner2022deep}. 


\paragraph{Granger Causality} Granger causality \citep{granger1969investigating} has been extensively used for temporal causal discovery. It is based on the idea that the series $\mX^j$ does not Granger cause $\mX^i$ if the history, $\mX^j_{<t}$, does not help the prediction of $X^i_t$ for some $t$ given the past of all other time series $\mX^k$ for $k\neq j,i$.
\begin{definition}[Granger Causality \citep{tank2018neural, lowe2022amortized}]
Given a multivariate stationary time series $\{\mX_t\}_{t=0}^T$ and a SEM $f_{i,t}$ defined as 
\begin{equation}
    X_t^i=f_{i,t}(\PaGst)+\epsilon_{t}^i,
\end{equation}
$\mX^j$ Granger causes $\mX^i$ if $\exists l\in [1,t]$ such that $X_{t-l}^j\in \PaGst$ and $f_{i,t}$ depends on $X_{t-l}^j$.
\label{def: Granger Causality}
\end{definition}
Granger causality is equivalent to causal relations for \emph{directed acyclic graph} (\emph{DAG}) if there are no latent confounders and instantaneous effects \citep{peters2013causal, peters2017elements}. Apart from the lack of instantaneous effects, it also ignore the history-dependent noise with independent $\epsilon_t^i$. 


\paragraph{Vector Auto-regressive Model}
Another line of research focuses on directly fitting the identifiable SEM to the observational data with instantaneous effects. One commonly-used approach is called vector auto-regression \citep{hyvarinen2010estimation, pamfil2020dynotears}:
\begin{equation}
X_t^i = \beta^i +\sum_{\tau=0}^K\sum_{j=1}^DB_{\tau,ji}X_{t-\tau}^j +\epsilon_{t}^i
    \label{eq: Vector Auto-regressive SEM}
\end{equation}
where $\beta^i$ is the offset, $\tau$ is the model lag, $\mB_{\tau}\in\R^{D\times D}$ is the weighted adjacency matrix specifying the connections at time $t-\tau$ (i.e.~if $B_{\tau,ji}=0$ means no connection from $X^j_{t-\tau}$ to $X^i_t$) and $\epsilon_{t}^i$ is the independent noise. Under these assumptions, the above linear SEM is structurally identifiable, which is a necessary condition for recovering the ground truth graph \citep{hyvarinen2010estimation, peters2013causal, pamfil2020dynotears}. However, the above linear SEM with independent noise variables is too restrictive to fulfil the requirements described in \cref{sec: introduction}. Therefore, the research question is how to design a structurally identifiable non-linear SEM with flexible history-dependent noise.

\section{\ModelName: Relationship learning with history dependent noise}
\label{sec: AR-DECI}
This section introduces the \ModelName{} model: \Cref{subsec: AR-DECI formulation} describes specific choices in the form of \ModelName{}'s SEM, allowing for history-dependent noise. \Cref{subsec: variational inference for AR-DECI} details how vartiaional inference can be leveraged to perform causal discovery with the proposed functional form of the SEM.
\subsection{Model formulation}
\label{subsec: AR-DECI formulation}
For a multivariate stationary time series $\{\mX_{t}\}_{t=0}^T$, we assume that their causal relations follow a temporal adjacency matrix $\mG_{0:K}$ with maximum lag $K$ where $\mG_{\tau\in[1,K]}$ specifies the lagged effects between $\mX_{t-\tau}$ and $\mX_t$, $\mG_0$ specifies the instantaneous parents. We define $G_{\tau,ij}=1$ if $X_{t-\tau}^i \rightarrow X_t^j$ and $0$ otherwise.
\footnote{In the following, we interchange the usage of the notation $\mG$ and $\mG_{0:K}$ for brevity.} 
We propose a novel functional causal model 
that incorporates non-linear relations, instantaneous effects, and flexible history-dependent noise, called \ModelName{}:
\begin{equation}
X_t^i = f_i(\PaGst, \PaGt) + g_i(\PaGst,\epsilon_{t}^i)
    \label{eq: SEM for AR-DECI}
\end{equation}
where $f_i$ is a general differentiable non-linear function, and $g_i$ is a differentiable transform s.t. the transformed noise has a proper density. 
Despite that \ModelName{} has an additive structure, 
our formulation offers much more flexibility in both functional relations and noise distributions compared to previous works \citep{pamfil2020dynotears, peters2013causal}. By placing few restrictions on $f_i, g_i$, \ModelName{} can capture functional non-linearity through $f_i$ and transform $\epsilon_{t}^i$ through a flexible function $g_i$, depending on $\PaGst$, to capture the history dependency of the additive noise. 

Next, we propose flexible functional designs for $f_i, g_i$, which must respect the relations encapsulated in $\mG$. Namely, if $X_{t-\tau}^j\notin \PaGst\cup \PaGt$, then ${\partial f_i}/{\partial X_{t-\tau}^j}=0$ and similarly for $g_i$. We design
\begin{equation}
    f_i(\PaGst,\PaGt)=\zeta_i\left(\sum_{\tau=0}^K\sum_{j=1}^DG_{\tau,ji}\ell_{\tau j}\left(X_{t-\tau}^j\right)\right)
    \label{eq: model design of AR-DECI}
\end{equation}
where $\zeta_i$ and $\ell_{\tau i}$ ($i\in [1,D]$ and $\tau\in [0,K]$) are neural networks. For efficient computation, we use weight sharing across nodes and lags: $\zeta_i(\cdot) = \zeta(\cdot, \vu_{0,i})$ and $\ell_{\tau j}(\cdot) = \ell(\cdot, \vu_{\tau,j})$, where $\vu_{\tau,i}$ is the trainable embedding for node $i$ at time $t-\tau$.

The design of $g_i$ needs to properly balance the flexibility and tractability of the transformed noise density for the sake of training. 
{We thus choose} a conditional normalizing flow, called conditional spline flow \citep{trippe2018conditional,durkan2019neural, pawlowski2020deep}, with a fixed Gaussian noise $\epsilon_{t}^i$ for all $t$ and $i$. The spline bin parameters are predicted using a hyper-network with a similar form to \cref{eq: model design of AR-DECI} to incorporate history dependency. The only difference is now $\tau$ is summed over $[1,K]$ to remove the instantaneous parents. 
Due to the invertibility of $g_i$, the noise likelihood conditioned on lagged parents is
\begin{equation}
    p_{g_i}(g_i(\epsilon_{t}^i)\vert \PaGst) = p_{\epsilon}(\epsilon_{t}^i)\left\vert\frac{\partial g_i^{-1}}{\partial \epsilon_{t}^i}\right\vert.
    \label{eq: AR-DECI conditional flow probability}
\end{equation}

\subsection{Variational Inference for \ModelName{}}
\label{subsec: variational inference for AR-DECI}
\ModelName{} adopts a Bayesian view of causal discovery \citep{heckerman2006bayesian}, which aims to learn a graph posterior distribution instead of inferring a single graph. For $N$ observed multivariate time series $\multitimeX{1},\ldots,\multitimeX{N}$, the joint likelihood of \ModelName{} is 
\begin{equation}
    p(\multitimeX{1},\ldots, \multitimeX{N},\mG) = p(\mG)\prod_{n=1}^Np_{\theta}(\multitimeX{n}\vert \mG)
    \label{eq: AR-DECI joint likelihood}
\end{equation}
where $\theta$ are the model parameters. Once fitted, the posterior $p(\mG|\multitimeX{1},\ldots,\multitimeX{N})$ incorporates the belief of the underlying causal relationships.

\paragraph{Graph Prior} When designing the graph prior, we combine three components: (1) DAG constraint; (2) graph sparseness prior; (3) domain-specific prior knowledge (optional). Inspired by the NOTEARS \citep{zheng2018dags, geffner2022deep, morales2021vicause}, we propose the following unnormalised prior
\begin{equation}
    p(\mG) \propto \exp\left(-\lambda_s \Vert\Graph[0]\Vert_{F}^2-\rho h^2(\singlegraph)-\alpha h (\singlegraph)-\lambda_p\Vert\Graph-\Graph^p\Vert_F^2\right)
    \label{eq: AR-DECI Graph Prior}
\end{equation}
where $h(\mG)=\tr(e^{\mG\odot \mG})-D$ is the DAG penalty proposed in \citep{zheng2018dags} and is 0 if and only if $\mG$ is a DAG; $\odot$ is the Hadamard product; $\mG^p$ is an optional domain-specific prior graph, {which can be used when partial domain knowledge is available}
; $\lambda_s$, $\lambda_p$ specify the strength of the graph sparseness and domain-specific prior terms respectively; $\alpha$, $\rho$ characterize the strength of the DAG penalty. Since the lagged connections specified in $\mG_{1:K}$ can only follow the direction of time, only the instantaneous part, $\mG_0$, can contain cycles. Thus, the DAG penalty is only applied to $\singlegraph$.

\paragraph{Variational Objective}
Unfortunately, the exact graph posterior $p(\mG|\multitimeX{1},\ldots,\multitimeX{N})$ is intractable due to the large combinatorial space of DAGs. To overcome this challenge, we adopt variational inference \citep{blei2017variational, zhang2018advances}, which uses a variational distribution $\vardist$ to approximate the true posterior. We choose $\vardist$ to be a product of independent Bernoulli distributions (refer to \cref{app: variational dist formulation} for details). The corresponding \emph{evidence lower bound} (\emph{ELBO}) is
\begin{equation}
\log p_\theta\left(\multitimeX{1},\ldots,\multitimeX{N}\right)\geq \underbrace{\E_{\vardist}\left[\sum_{n=1}^N\log p_\theta(\multitimeX{n}\vert \mG)+\log p(\mG)\right] + H(\vardist)}_{\text{ELBO}(\theta,\phi)}
    \label{eq: AR-DECI ELBO}
\end{equation}
where $H(\vardist)$ is the entropy of $\vardist$. 
From the causal Markov assumption and auto-regressive nature, we can further simplify
\begin{equation}
\log p_\theta(\multitimeX{n}|\mG) = \sum_{t=0}^T\sum_{i=1}^D\log p_\theta(X_t^{i,(n)}|\PaGst, \PaGt)
    \label{eq: AR-DECI forward likelihood}
\end{equation}
and from \ModelName{}'s functional form (\cref{eq: SEM for AR-DECI}) proposed in \cref{subsec: AR-DECI formulation}
\begin{equation}
\log p_\theta(X_t^{i,(n)}\vert \PaGst, \PaGt)=\log p_{g_i}\left(z_t^{i,(n)}\vert \PaGst\right)
    \label{eq: AR-DECI individual likelihood}
\end{equation}
where $z_t^{i,(n)} = X_t^{i,(n)}-f_i(\PaGst,\PaGt)$ and $p_{g_i}$ is defined in \cref{eq: AR-DECI conditional flow probability} (\cref{app: ELBO derivation} for details).
The parameters $\theta$, $\phi$ are learned by maximizing the ELBO, where the Gumbel-softmax gradient estimator is used for $\phi$ \citep{jang2016categorical, maddison2016concrete}. 
We also leverage augmented Lagrangian training \citep{hestenes1969multiplier,andreani2008augmented}, similar as \citet{geffner2022deep}, to anneal $\alpha, \rho$ in the prior to make sure \ModelName{} only produces DAGs (refer to Appendix B.1 in \citet{geffner2022deep}). 
Once \ModelName{} is fitted, the temporal causal graph can be inferred by $\mG\sim q_\phi(\mG)$. 

\paragraph{Treatment effect estimation}
{As \ModelName{} learns the causal graph and the functional relationship simultaneously, }
our model can be extended for causal inference tasks such as treatment effect estimation \citep{geffner2022deep}. See \cref{app: Treatment effect estimation} for details. 

\section{Theoretical Considerations}
\label{sec: theoretical considerations}
In this section, we focus on the theoretical guarantees of \ModelName{} including (1) structural identifiability and (2) the soundness of the proposed variational objective. Together, they guarantee the validity of \ModelName{} as a causal discovery method. In the end, we clarify its relations to existing works. 
\subsection{Structural Identifiability}
\label{subsec: AR-DECI structure identifiability}
One of the key challenges for causal discovery with a flexible functional relationship is to show the structural identifiability. Namely, we cannot find two different graphs that induce the same joint likelihood from the proposed functional causal model. In the following, we present a theorem for \ModelName{} that summarizes our main theoretical contribution. 

\begin{theorem}[Identifiability of \ModelName{}]
Assuming \ModelName{} satisfies the \emph{causal Markov property, causal minimality, causal sufficiency and the induced joint likelihood has a proper density} (see \cref{app: Structure Identifiability} for details), and we further assume (1) all functions and induced distributions of \ModelName{} are third-order differentiable; (2) function $f_i$ is \emph{non-linear} and \emph{not invertible} w.r.t. any nodes in $\PaGt$; (3) the double derivative $(\log p_{g_i}(g_i(\epsilon_{t}^i)\vert \PaGst))''$ w.r.t $\epsilon_{t}^i$ is zero at most at some discrete points, then \ModelName{} defined in \cref{eq: SEM for AR-DECI} is structural identifiable for \emph{both bivariate and multivariate time series}. 
\label{thm: identifiability of AR-DECI}
\end{theorem}
\begin{sproof}
This theorem is a summary of a collection of theorems proved in \cref{app: Structure Identifiability}. The strategy is 
instead of directly proving the identifiability of \ModelName{}, we provide identifiability conditions for a general temporal SEMs, followed by showing a generalization of \ModelName{} satisfies these conditions. The identifiability of \ModelName{} directly follows from it. 
\paragraph{Prove bivariate identifiability conditions for general temporal SEMs} 
The first step is to prove the bivariate identifiability conditions that a general temporal SEM (\cref{eq: Temporal SEM}) should satisfy (refer to \cref{thm: general identifiability} in \cref{subapp: general identifability conditions}). In a nutshell, we proved the functional causal model is bivariate identifiable if (1) the model for initial conditions is identifiable; (2) the model is \textbf{identifiable w.r.t. instantaneous parents}. Remarkably, (2) implies we only need to pay attention to instantaneous parents for identifiability, and opens the door for flexible lagged parent dependency.
This theorem assumes \emph{causal Markov, minimality, sufficiency and proper density} assumptions. 

\paragraph{Identifiability of history-dependent post non-linear model} Next, we propose a novel generalization of \ModelName{}, called \emph{history-dependent PNL}. \Cref{thm: bivariate identifiable history dependent PNL} and \cref{corollary: Validity of neural network} in \cref{subapp: Identifiability of history dependent PNL} prove it is {bivariate identifiable} w.r.t. instantaneous parents (i.e.~satisfy the conditions in \cref{thm: general identifiability}) with additional assumptions (1), (2) and (3) in \cref{thm: identifiability of AR-DECI}. The history-dependent PNL is defined as
\[
X_t^i = \nu_{it}\left(f_{it}\left(\PaGst[i], \PaGt[i]\right)+g_{it}\left(\PaGst,\epsilon_{it}\right), \PaGst\right),
\]
where $\nu$ is invertible w.r.t.~the first argument. 
The bivariate identifiability of \ModelName{} directly follows from this, since \ModelName{} is a special case with $\nu$ being the identity mapping.
\paragraph{Generalization to multivariate case} In the end, inspired by \citet{peters2012identifiability}, we prove the above bivariate identifiability can be generalized to the multivariate case. Refer to \cref{thm: multivariate identifiability} in \cref{appsubsec: generalizing to multivariate time series} for details. 

\end{sproof}
\subsection{Validity of variational objective and relations to other methods}
Next, we show the validity of the variational objective (\cref{eq: AR-DECI ELBO}) in the sense that optimizing it can lead to the ground truth graph. Remarkably, Theorem 1 in \citet{geffner2022deep} justifies the validity of the variational objective under the same set of assumptions as \ModelName{}. 
\begin{theorem}[Validity of variational objective \citep{geffner2022deep}]
Assuming the conditions in \cref{thm: identifiability of AR-DECI} are satisfied, and we further assume that there is no model misspecification, then the solution $(\theta', q'_\phi(\mG))$ from optimizing \cref{eq: AR-DECI ELBO} with infinite data satisfies $q'_\phi(\mG)=\delta(\mG=\mG')$, where $\mG'$ is a unique graph. In particular, $\mG'=\mG^*$ and $p_{\theta'}(\singletimeX;\mG')=p(\singletimeX;\mG^*)$, where $\mG^*$ is the ground truth graph and $p(\singletimeX;\mG^*)$ is the true data generating distribution.
\label{thm: validity of variational objective}
\end{theorem}

\paragraph{Relation to other methods} 
Many previous works of using functional causal model for causal time series discovery \citep{hyvarinen2010estimation, pamfil2020dynotears, tank2018neural, peters2013causal} are closely related to \ModelName{}. Since \ModelName{} incorporates history-dependent noise with flexible non-linear functional relations, it is the most flexible member of this family. Refer to \cref{appsec: relation to other methods} for details.

\section{Related Work}
\label{sec: related work}
Discovering causal relationships from time series has been a popular research question for several decades now. \citet{assaad2022survey} provides a comprehensive overview of causal discovery method for time series. In a nutshell, there are three main categories. 
The first category is Granger causality, where this field can be further split into (1) vector auto-regressive methods \citep{wu2020discovering, shojaie2010discovering, siggiridou2015granger, amornbunchornvej2019variable} and (2) deep learning based approaches \citep{lowe2022amortized, tank2018neural, bussmann2021neural, dang2018seq2graph, xu2019scalable}. Despite recent advances, all Granger causality methods cannot handle instantaneous effects, which can be observed due to the aggregation effect in a slow-sampling system. 
Additionally, they also assume a fixed noise distribution without history dependency. 

Using SEMs for time series discovery can mitigate the aforementioned two problems. VARLiNGaM \citep{hyvarinen2010estimation} extends the identifiability theory of linear non-Gaussian ANM \citep{shimizu2006linear} to vector auto-regression for modelling time series data. DYNOTEARS \citep{pamfil2020dynotears} leverages the recently proposed NOTEARS framework \citep{zheng2018dags} to continuously relax the DAG constraints for fully differentiable DAG structure learning. However, the above approach is still limited to linear functional forms. TiMINo \citep{peters2013causal} provides a general theoretical framework for temporal causal discovery with SEMs. Our theory leverages some of their proof techniques. Unfortunately, all the aforementioned methods assume no history dependency for the noise. On the other hand, \ModelName{} can model (1) non-linear function relations; (2) instantaneous effect; (3) and history-dependent noise at the same time.

The third category is constraint-based approaches based on conditional independence tests. Due to its non-parametric nature, it can handle history-dependent noise. PCMCI \citep{runge2019detecting} combines PC \citep{spirtes2000causation} and the momentary conditional independence test to discover the lagged parents from time series.
PCMCI$^+$ \citep{runge2018causal, runge2020discovering} further extends PCMCI to infer both lagged and instantaneous effects. CD-NOD \citep{huang2020causal} has recently been proposed to handle non-stationary heterogeneous data, where the data distribution can shift across time. Despite their generality, they can only infer MECs; cannot learn the explicit functional forms between variables; and require a stronger assumption than minimality (i.e.~faithfulness).

\section{Experiments}
\label{sec: Experiments}
\subsection{Synthetic data}
\label{subsec: synthetic data}
We evaluate our method on a large set of synthetically generated datasets with known causal graphs. We use the main body of this paper to present the overall performance of our method compared to relevant baselines and one ablation study on the robustness to lag mismatch. In \cref{sec:app_syn_results}, we conduct extensive analysis, including (1) on different graph type; (2) ablation on history-dependency; (3) ablation study on instantaneous effect.
This set of datasets are generated by various settings (e.g.~type of graphs, instantaneous/no instantaneous effect, etc.). 5 datasets are generated for each combination of settings with different seeds, yielding 160 datasets in total.
In order to comprehensively test \ModelName{}'s robustness, we deliberately generated  
$75\%$ of the datasets that mismatch the \ModelName{} configurations. 
Details of the data generation can be found in \cref{sec:app_syn_data}.

We compare \ModelName{} to a wide range of baselines, including VARLiNGaM \citep{hyvarinen2010estimation}, PCMCI$^+$\citep{runge2020discovering} and DYNOTEARS \citep{pamfil2020dynotears}.
PCMCI only outputs Markov equivalence classes (MECs). We resolve this by enumerating all DAGs in the MEC. For details on the methods, see \cref{sec:app_syn_methods}.
Additionally, we include two variants of \ModelName{}: (1) \ModelName{}+g, where an independent Gaussian noise is used; (2) \ModelName{}+s, where Gaussian $\epsilon_i$ is transformed by an independent spline.

\Cref{fig:synth_f1_avg_all} presents the F$_1$ score of the lagged, instantaneous and temporal adjacency matrix of all methods aggregated over all datasets\footnote{It is worth noting that we run each method on 40 different dataset settings for all possible numbers of nodes.}, denoted as 'Lag', 'Inst.' and 'Temporal', respectively.
\ModelName{} achieves overall competitive or the best performance in terms of the full temporal adjacency matrix across all possible datasets, especially for lower dimensions.
Comparing \ModelName{}'s lagged discovery to its two variants, the better score indicates the history-dependent noise is useful to the lagged graph discovery, contributing to the better overall F$_1$ performance (\cref{sec:app_syn_results} for ablation: with/without history dependency). 

Despite of the strong performance from PCMCI$^{+}$, it can only identify the graph up to MECs without explicit functional relations. Computationally, PCMCI$^{+}$ exceeds the maximum training time of \textbf{1 week} on 40 nodes (see \cref{sec:app_syn_results}), suggesting its computation bottleneck in high dimensions.

DYNOTEARS achieves inferior results in general due to limited modelling power from the linear nature. This is much clearer in high dimensions due to the increasing difficulty of the problem.

\begin{figure}[htb]
    \centering
    \includegraphics[width=\linewidth]{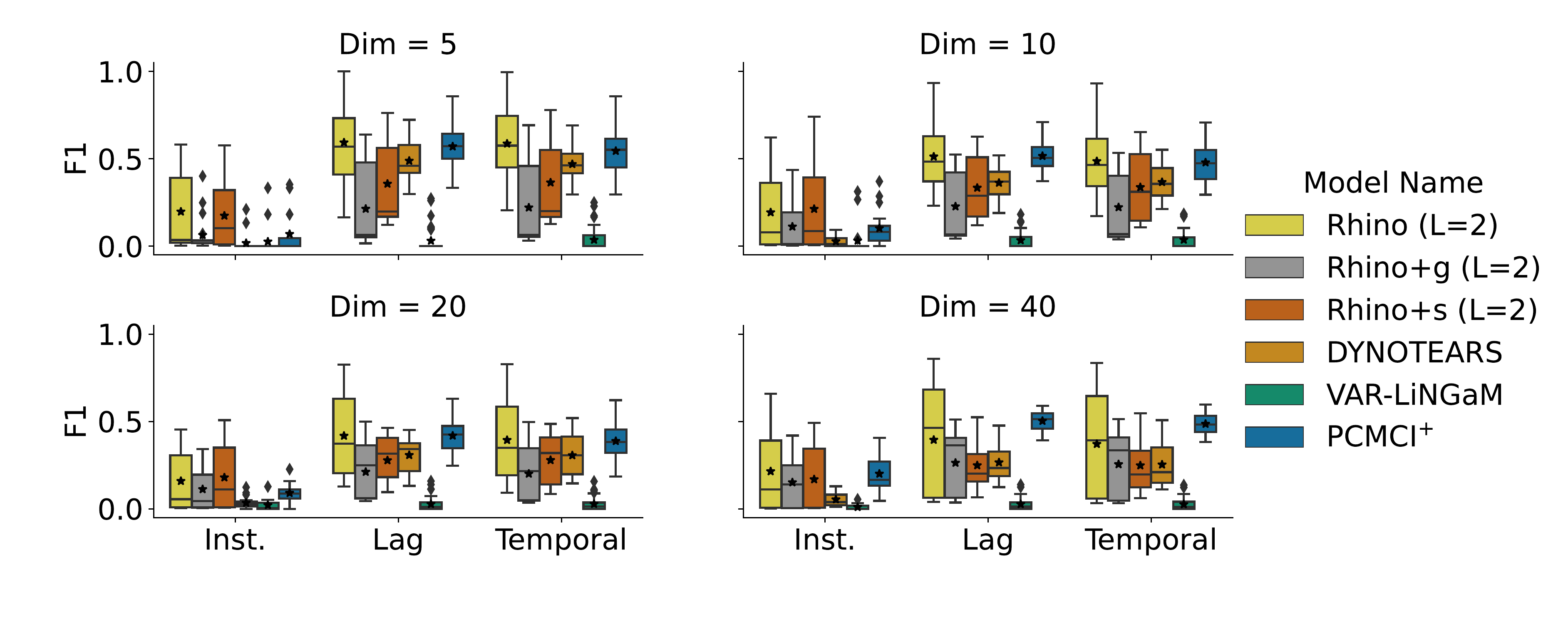}
    \vspace{-2em}
    \caption{$F_1$-scores of \ModelName{} (light yellow) compared to all baseline methods. The different subplots show the performance for the different number of nodes in the datasets. `L=2' refers to models with lag 2.}
    \label{fig:synth_f1_avg_all}
\end{figure}

We explore the behaviour of \ModelName{} with different lag parameters other than the ground truth lag 2. From \cref{tab:rhino_lags}, worse training log-likelihoods suggest that \ModelName{} with insufficient history ($\mathrm{lag}=1$) is unable to correctly model the data and this leads to a decrease in F$_1$ scores. Interestingly, \ModelName{} is also robust with longer lags. Despite of the slightly better likelihood ($\mathrm{lag}=3$), it achieves comparable performance to the model with the correct lag. Also, from their similar F$_1$ Lag score, it suggests the extra adjacency matrix is nearly empty.   
\begin{table}[htb]
    \centering
    \begin{tabular}{llrrr}
\toprule
Dim & & Rhino (L=1) & Rhino (L=2) & Rhino (L=3) \\
\midrule
5  & F$_1$ Inst. &    $0.11 \pm 0.17$ &    $0.20 \pm 0.22$ &    $0.21 \pm 0.23$ \\
   & F$_1$ Lag &    $0.28 \pm 0.13$ &    $0.59 \pm 0.22$ &    $0.57 \pm 0.24$ \\
   & F$_1$ Temporal &    $0.34 \pm 0.12$ &    $0.59 \pm 0.20$ &    $0.56 \pm 0.22$ \\
   & LL &   $-4.14 \pm 1.63$ &   $-3.83 \pm 1.62$ &   $-3.75 \pm 1.64$ \\
 \midrule
10 & F$_1$ Inst. &    $0.13 \pm 0.17$ &    $0.19 \pm 0.23$ &    $0.19 \pm 0.22$ \\
   & F$_1$ Lag &    $0.26 \pm 0.08$ &    $0.51 \pm 0.17$ &    $0.48 \pm 0.19$ \\
   & F$_1$ Temporal &    $0.28 \pm 0.11$ &    $0.49 \pm 0.18$ &    $0.45 \pm 0.20$ \\
   & LL &   $-7.97 \pm 2.09$ &   $-7.21 \pm 2.22$ &   $-7.01 \pm 1.91$ \\
\midrule
20 & F$_1$ Inst. &    $0.15 \pm 0.16$ &    $0.16 \pm 0.17$ &    $0.18 \pm 0.19$ \\
   & F$_1$ Lag &    $0.24 \pm 0.12$ &    $0.42 \pm 0.22$ &    $0.40 \pm 0.21$ \\
   & F$_1$ Temporal &    $0.25 \pm 0.13$ &    $0.39 \pm 0.22$ &    $0.37 \pm 0.21$ \\
   & LL &  $-15.62 \pm 3.16$ &  $-14.70 \pm 2.87$ &  $-14.72 \pm 2.82$ \\
\midrule
40 & F$_1$ Inst. &    $0.13 \pm 0.16$ &    $0.22 \pm 0.23$ &    $0.18 \pm 0.21$ \\
   & F$_1$ Lag &    $0.20 \pm 0.18$ &    $0.40 \pm 0.31$ &    $0.34 \pm 0.30$ \\
   & F$_1$ Temporal &    $0.20 \pm 0.18$ &    $0.37 \pm 0.30$ &    $0.32 \pm 0.29$ \\
   & LL &  $-31.44 \pm 5.16$ &  $-30.10 \pm 4.71$ &  $-30.20 \pm 4.74$ \\
\bottomrule
\end{tabular}
    \caption{Comparison of the causal discovery performance of \ModelName{} with different lag-parameters ($L \in [1, 3]$). Apart from the 3 F$_1$ scores,  LL shows the log-likelihood of the training data.}
    \label{tab:rhino_lags}
\end{table}

\subsection{DREAM3 Gene network}
\label{subsec: DREAM3}
In this section, we evaluate \ModelName{} performance with a real-world biology benchmark called \emph{DREAM3} \citep{prill2010towards, marbach2009generating}. These datasets are often used to evaluate Granger causality \citep{khanna2019economy, tank2018neural, nauta2019causal, bussmann2021neural} but recently adopted for SEM-based method \citep{pamfil2020dynotears}. The dataset consists \emph{in silico} measurements of gene expression levels for $5$ different networks. Each network contains $d=100$ genes. Each time series represents a perturbation trajectory with time length $T=21$. For each network, $46$ perturbation trajectories are recorded. The goal is to infer the causal structure of each network. We use the \emph{area under the ROC curve} (AUROC) as the performance metric. We consider the same baselines as in the synthetic experiments (i.e.~DYNOTEARS and PCMCI$^+$) without \emph{VARLiNGaM} since its default implementation fails when the number of variables ($d=100$) is greater than the series length ($T=21$). Additionally, we also consider relevant Granger causality methods, including \emph{cMLP}, \emph{cLSTM} \citep{tank2018neural}; \emph{TCDF}\citep{nauta2019causal}; \emph{SRU} and \emph{eSRU} \citep{khanna2019economy}). Their corresponding results are directly cited from \citet{khanna2019economy}. \Cref{subsubapp: DREAM3 Hyperparams} specifies \ModelName{} hyperparameters. Since the ground truth graph is a summary graph (see \cref{def: summary graph} in \cref{subsubapp: DREAM3 adj matrix aggregation}), 
\cref{subsubapp: DREAM3 adj matrix aggregation} details about the post-processing step on aggregating temporal graph to summary graph for \ModelName{}, DYNOTEARS and PCMCI$^+$.

\begin{table}[htb]
\centering
\begin{tabular}{lccccc}
\toprule
Method  & E.Coli 1 & E.Coli 2 & Yeast 1 & Yeast 2 & Yeast 3 \\
\midrule
cMLP    & 0.644    & 0.568    & 0.585   & 0.506   & 0.528   \\
cLSTM   & 0.629    & 0.609    & 0.579   & 0.519   & 0.555   \\
TCDF    & 0.614    & 0.647    & 0.581   & 0.556   & 0.557   \\
SRU     & 0.657    & 0.666    & 0.617   & 0.575   & 0.55    \\
eSRU    & 0.66     & 0.629    & 0.627   & 0.557   & 0.55    \\
DYNO. &0.590 &0.547 &0.527 &0.526 &0.510\\
PCMCI$^+$ &$0.530\pm 0.002$&$0.519\pm 0.002$ &$0.530\pm 0.003$ &$0.510\pm 0.001$ &$0.512\pm 0$ \\
{\ModelName{}}+g &\textbf{0.673}$\bm{\pm 0.013}$ &0.665$\bm{\pm 0.009}$ &\textbf{0.659}$\bm{\pm 0.005}$ &\textbf{0.598}$\bm{\pm 0.004}$ &\textbf{0.588}$\bm{\pm 0.005}$ \\
\ModelName{} & \textbf{0.685$\bm{\pm {0.003}}$}         & \textbf{0.680}$\bm{\pm 0.007}$         &\textbf{0.664$\bm{\pm 0.006}$}         & 0.585${\pm 0.004}$        & 0.567$\pm 0.003$\\        
\bottomrule
\end{tabular}
\caption{The AUROC of the summary graph adjacency matrix for 5 datasets in DREAM3, where self-connections are ignored. DYNO. means DYNOTEARS. For \ModelName{} and PCMCI$^+$, the mean results with standard error are reported by averaging over $5$ runs. \citet{khanna2019economy} only reported the single-run results for baselines.}
\label{table: Exp DREAM3 AUROC}
\end{table}

\cref{table: Exp DREAM3 AUROC} demonstrates the AUROC of the summary graph inferred after training. It is clear that \ModelName{} and its variant outperform all other methods. Although \ModelName{} is not formulated to solve the summary graph discovery, it shows a clear advantage compared to the state-of-the-art Granger causality. Thus, \ModelName{} can be used to infer either temporal or summary graph depending on users' needs. 

By inspecting the hyperparameters of \ModelName{} in \cref{subsubapp: DREAM3 Hyperparams}, instantaneous effects seem to provide no obvious help for in these datasets. It suggests the recording intervals are fast enough to avoid any aggregation effect. This explains why the Granger causality can also perform reasonably well. 

Unlike the strong performances of DYNOTEARS and PCMCI$^+$ in synthetic experiments, they perform poorly in DREAM3. The linear nature of DYNOTEARS seems to harm its performance drastically. PCMCI$^+$ suffers from the low independence test power under small training data.


Another interesting ablation is to compare with \ModelName{}+g, which performs on par with \ModelName{} and achieves better scores on $2$ out of $5$ datasets. Although we have no access to the true noise mechanism, we suspect that the added noise is not history-dependent and highly likely to be Gaussian. Despite the model mismatch, \ModelName{} is still one of the best methods for this problem. This further strengthens our belief in the robustness of our model under different setups.

\subsection{Netsim Brain Connectivity}
\label{subsec: Netsim}
In this section, we evaluate \ModelName{} using \emph{blood oxygenation level dependent} (BOLD) imaging data, which has also been used as a benchmark for temporal causal discovery \citep{lowe2022amortized, khanna2019economy, assaad2022survey}. Each time series represents the BOLD signal simulated for a human subject, which describes $d=15$ different regions in the brain. Each series contains $T=200$ timestamps. The goal of this task is to infer the connectivity between different brain regions. We assume that different human subjects share the same connectivity. We only use the data from human subject $2-6$ in \emph{Sim-3.mat} from \url{https://www.fmrib.ox.ac.uk/datasets/netsim/index.html} and also include self-connections during evaluation.
We use the same set of baselines as DREAM3 (\cref{subsec: DREAM3}) plus VARLiNGaM. 
\cref{subsubapp: Netsim hyperparams} describes hyperparameter settings. 


\begin{wraptable}[20]{r}{0.50\textwidth}
\centering
\vspace{-0.5cm}
\begin{tabular}{lc}
\toprule
Method  & AUROC \\
\midrule
cMLP    & 0.93  \\
cLSTM   & 0.83  \\
TCDF    & 0.91  \\
SRU     & 0.80  \\
eSRU    & 0.88  \\ 
DYNO. & 0.90\\
PCMCI$^+$ & $0.83\pm 0$ \\
VARLiNGaM & $0.84\pm 0$\\
\midrule
\ModelName{}+g & $0.974\pm 0.002$\\
{\ModelName{}}+NoInst. & $0.93\pm 0.006$ \\
\ModelName &  $\bm{0.99\pm 0.001}$     \\
\bottomrule
\end{tabular}
\caption{The AUROCs of the summary graph. \ModelName{}+NoInst is the \ModelName{} without the instantaneous effects. For \ModelName{}, VARLiNGaM, PCMCI$^+$, the results are obtained by averaging over $5$ different runs. }
\label{tab: Exp Netsim AUROC}
\end{wraptable}

\cref{tab: Exp Netsim AUROC} shows the AUROCs for different methods. Remarkably, the proposed \ModelName{} and its variants achieve significantly better AUROC compared to the baselines. Especially, \ModelName{} obtains nearly optimal AUROC, demonstrating its robustness to the small dataset and good balances between true and false positive rates (see \cref{app: netsim AUROC metric}). By comparing \ModelName{} and \ModelName{}+NoInst., we conclude that modelling instantaneous effects is important in real application, indicating the sampling interval is not frequent enough to explain everything as lagged effects. This can be double confirmed by comparing \ModelName{}+NoInst with Granger causality, where it performs on par with the state-of-the-art baseline when disabling the instantaneous effect. Last but not least, by comparing \ModelName{}+g with \ModelName{}, we find that history-dependent noise is also helpful in this dataset.


\section{Conclusion}
\label{sec:conclusion}

Inferring temporal causal graphs from observational time series is an important task in many scientific fields. Especially, some applications (e.g.~education, climate science, etc.) require the modelling of non-linear relationships; instantaneous effects and history-dependent noise distributions at the same time. Previous works fail to offer an appropriate solution for all three requirements. Motivated by this, we propose \ModelName{}, which combines vector auto-regression with deep learning and variational inference to perform causal temporal relationship learning with all three requirements. Theoretically, we prove the structural identifiability of \ModelName{} with flexible history-dependent noise, and clarify its relations to existing works. Empirical evaluations demonstrate its superior performance and robustness when \ModelName{} is misspecified, and the advantages of history-dependent noise mechanisms. This opens an exciting route of extending \ModelName{} to handle non-stationary time-series and unobserved confounders in future work.

\section{Reproducibility Statement}
\paragraph{Theoretical Contributions}
The main theoretical contribution is summarized in \cref{thm: identifiability of AR-DECI}. This theorem is the result from a collection of theorems proved in \cref{app: Structure Identifiability}. In \cref{app: Structure Identifiability}, we detailed the fundamental assumptions (\cref{assumption 1}-\cref{assumption 5}) required for the all theorems. The theorem-specific assumptions are mentioned in the statement of the theorem. To ease the understanding of the proof, we also provide the skecth of proof in \cref{thm: identifiability of AR-DECI}. Since \cref{thm: validity of variational objective} is directly cited from \citep{geffner2022deep} without major modification, the proof can be found in Appendix A in \citet{geffner2022deep}.

\paragraph{Empirical Evaluations} For synthetic, DREAM3 and Netsim experiments, we listed the hyperparameters in \cref{sec:app_syn_methods}, \cref{subsubapp: DREAM3 Hyperparams} and \cref{subsubapp: Netsim hyperparams}, respectively. \cref{sec:app_syn_data} explains the synthetic data generation. For DREAM3 and Netsim, the dataset can be found in the public github repo \url{https://github.com/sakhanna/SRU_for_GCI/tree/master/data}. The post processing steps for DREAM3 and Netsim evaluations are described in \cref{subsubapp: DREAM3 adj matrix aggregation}.
\bibliography{iclr2023_conference}

\begin{thebibliography}{50}
\providecommand{\natexlab}[1]{#1}
\providecommand{\url}[1]{\texttt{#1}}
\expandafter\ifx\csname urlstyle\endcsname\relax
  \providecommand{\doi}[1]{doi: #1}\else
  \providecommand{\doi}{doi: \begingroup \urlstyle{rm}\Url}\fi

\bibitem[Amornbunchornvej et~al.(2019)Amornbunchornvej, Zheleva, and
  Berger-Wolf]{amornbunchornvej2019variable}
Chainarong Amornbunchornvej, Elena Zheleva, and Tanya~Y Berger-Wolf.
\newblock Variable-lag granger causality for time series analysis.
\newblock In \emph{2019 IEEE International Conference on Data Science and
  Advanced Analytics (DSAA)}, pp.\  21--30. IEEE, 2019.

\bibitem[Andreani et~al.(2008)Andreani, Birgin, Mart{\'\i}nez, and
  Schuverdt]{andreani2008augmented}
Roberto Andreani, Ernesto~G Birgin, Jos{\'e}~Mario Mart{\'\i}nez, and
  Mar{\'\i}a~Laura Schuverdt.
\newblock On augmented lagrangian methods with general lower-level constraints.
\newblock \emph{SIAM Journal on Optimization}, 18\penalty0 (4):\penalty0
  1286--1309, 2008.

\bibitem[Assaad et~al.(2022)Assaad, Devijver, and Gaussier]{assaad2022survey}
Charles~K Assaad, Emilie Devijver, and Eric Gaussier.
\newblock Survey and evaluation of causal discovery methods for time series.
\newblock \emph{Journal of Artificial Intelligence Research}, 73:\penalty0
  767--819, 2022.

\bibitem[Berzuini et~al.(2012)Berzuini, Dawid, and
  Bernardinell]{berzuini2012causality}
Carlo Berzuini, Philip Dawid, and Luisa Bernardinell.
\newblock \emph{Causality: Statistical perspectives and applications}.
\newblock John Wiley \& Sons, 2012.

\bibitem[Blei et~al.(2017)Blei, Kucukelbir, and McAuliffe]{blei2017variational}
David~M Blei, Alp Kucukelbir, and Jon~D McAuliffe.
\newblock Variational inference: A review for statisticians.
\newblock \emph{Journal of the American statistical Association}, 112\penalty0
  (518):\penalty0 859--877, 2017.

\bibitem[Bussmann et~al.(2021)Bussmann, Nys, and Latr{\'e}]{bussmann2021neural}
Bart Bussmann, Jannes Nys, and Steven Latr{\'e}.
\newblock Neural additive vector autoregression models for causal discovery in
  time series.
\newblock In \emph{International Conference on Discovery Science}, pp.\
  446--460. Springer, 2021.

\bibitem[Dang et~al.(2018)Dang, Shah, and Zerfos]{dang2018seq2graph}
Xuan-Hong Dang, Syed~Yousaf Shah, and Petros Zerfos.
\newblock seq2graph: Discovering dynamic dependencies from multivariate time
  series with multi-level attention.
\newblock \emph{arXiv preprint arXiv:1812.04448}, 2018.

\bibitem[Dolatabadi et~al.(2020)Dolatabadi, Erfani, and
  Leckie]{dolatabadi2020invertible}
Hadi~Mohaghegh Dolatabadi, Sarah Erfani, and Christopher Leckie.
\newblock Invertible generative modeling using linear rational splines.
\newblock In \emph{International Conference on Artificial Intelligence and
  Statistics}, pp.\  4236--4246. PMLR, 2020.

\bibitem[Durkan et~al.(2019)Durkan, Bekasov, Murray, and
  Papamakarios]{durkan2019neural}
Conor Durkan, Artur Bekasov, Iain Murray, and George Papamakarios.
\newblock Neural spline flows.
\newblock \emph{Advances in neural information processing systems}, 32, 2019.

\bibitem[Geffner et~al.(2022)Geffner, Antoran, Foster, Gong, Ma, Kiciman,
  Sharma, Lamb, Kukla, Pawlowski, et~al.]{geffner2022deep}
Tomas Geffner, Javier Antoran, Adam Foster, Wenbo Gong, Chao Ma, Emre Kiciman,
  Amit Sharma, Angus Lamb, Martin Kukla, Nick Pawlowski, et~al.
\newblock Deep end-to-end causal inference.
\newblock \emph{arXiv preprint arXiv:2202.02195}, 2022.

\bibitem[Granger(1969)]{granger1969investigating}
Clive~WJ Granger.
\newblock Investigating causal relations by econometric models and
  cross-spectral methods.
\newblock \emph{Econometrica: journal of the Econometric Society}, pp.\
  424--438, 1969.

\bibitem[Guo et~al.(2020)Guo, Cheng, Li, Hahn, and Liu]{guo2020survey}
Ruocheng Guo, Lu~Cheng, Jundong Li, P~Richard Hahn, and Huan Liu.
\newblock A survey of learning causality with data: Problems and methods.
\newblock \emph{ACM Computing Surveys (CSUR)}, 53\penalty0 (4):\penalty0 1--37,
  2020.

\bibitem[Heckerman et~al.(2006)Heckerman, Meek, and
  Cooper]{heckerman2006bayesian}
David Heckerman, Christopher Meek, and Gregory Cooper.
\newblock A bayesian approach to causal discovery.
\newblock In \emph{Innovations in Machine Learning}, pp.\  1--28. Springer,
  2006.

\bibitem[Hestenes(1969)]{hestenes1969multiplier}
Magnus~R Hestenes.
\newblock Multiplier and gradient methods.
\newblock \emph{Journal of optimization theory and applications}, 4\penalty0
  (5):\penalty0 303--320, 1969.

\bibitem[Hoyer et~al.(2008)Hoyer, Janzing, Mooij, Peters, and
  Sch{\"o}lkopf]{hoyer2008nonlinear}
Patrik Hoyer, Dominik Janzing, Joris~M Mooij, Jonas Peters, and Bernhard
  Sch{\"o}lkopf.
\newblock Nonlinear causal discovery with additive noise models.
\newblock \emph{Advances in neural information processing systems}, 21, 2008.

\bibitem[Huang et~al.(2020)Huang, Zhang, Zhang, Ramsey, Sanchez-Romero,
  Glymour, and Sch{\"o}lkopf]{huang2020causal}
Biwei Huang, Kun Zhang, Jiji Zhang, Joseph~D Ramsey, Ruben Sanchez-Romero,
  Clark Glymour, and Bernhard Sch{\"o}lkopf.
\newblock Causal discovery from heterogeneous/nonstationary data.
\newblock \emph{J. Mach. Learn. Res.}, 21\penalty0 (89):\penalty0 1--53, 2020.

\bibitem[Hyv{\"a}rinen et~al.(2010)Hyv{\"a}rinen, Zhang, Shimizu, and
  Hoyer]{hyvarinen2010estimation}
Aapo Hyv{\"a}rinen, Kun Zhang, Shohei Shimizu, and Patrik~O Hoyer.
\newblock Estimation of a structural vector autoregression model using
  non-gaussianity.
\newblock \emph{Journal of Machine Learning Research}, 11\penalty0 (5), 2010.

\bibitem[Jang et~al.(2016)Jang, Gu, and Poole]{jang2016categorical}
Eric Jang, Shixiang Gu, and Ben Poole.
\newblock Categorical reparameterization with gumbel-softmax.
\newblock \emph{arXiv preprint arXiv:1611.01144}, 2016.

\bibitem[Khanna \& Tan(2019)Khanna and Tan]{khanna2019economy}
Saurabh Khanna and Vincent~YF Tan.
\newblock Economy statistical recurrent units for inferring nonlinear granger
  causality.
\newblock \emph{arXiv preprint arXiv:1911.09879}, 2019.

\bibitem[Khemakhem et~al.(2021)Khemakhem, Monti, Leech, and
  Hyvarinen]{khemakhem2021causal}
Ilyes Khemakhem, Ricardo Monti, Robert Leech, and Aapo Hyvarinen.
\newblock Causal autoregressive flows.
\newblock In \emph{International conference on artificial intelligence and
  statistics}, pp.\  3520--3528. PMLR, 2021.

\bibitem[Kingma \& Ba(2014)Kingma and Ba]{kingma2014adam}
Diederik~P Kingma and Jimmy Ba.
\newblock Adam: A method for stochastic optimization.
\newblock \emph{arXiv preprint arXiv:1412.6980}, 2014.

\bibitem[Lauritzen(1996)]{lauritzen1996graphical}
Steffen~L Lauritzen.
\newblock \emph{Graphical models}, volume~17.
\newblock Clarendon Press, 1996.

\bibitem[L{\"o}we et~al.(2022)L{\"o}we, Madras, Zemel, and
  Welling]{lowe2022amortized}
Sindy L{\"o}we, David Madras, Richard Zemel, and Max Welling.
\newblock Amortized causal discovery: Learning to infer causal graphs from
  time-series data.
\newblock In \emph{Conference on Causal Learning and Reasoning}, pp.\
  509--525. PMLR, 2022.

\bibitem[Maddison et~al.(2016)Maddison, Mnih, and Teh]{maddison2016concrete}
Chris~J Maddison, Andriy Mnih, and Yee~Whye Teh.
\newblock The concrete distribution: A continuous relaxation of discrete random
  variables.
\newblock \emph{arXiv preprint arXiv:1611.00712}, 2016.

\bibitem[Marbach et~al.(2009)Marbach, Schaffter, Mattiussi, and
  Floreano]{marbach2009generating}
Daniel Marbach, Thomas Schaffter, Claudio Mattiussi, and Dario Floreano.
\newblock Generating realistic in silico gene networks for performance
  assessment of reverse engineering methods.
\newblock \emph{Journal of computational biology}, 16\penalty0 (2):\penalty0
  229--239, 2009.

\bibitem[Moraffah et~al.(2021)Moraffah, Sheth, Karami, Bhattacharya, Wang,
  Tahir, Raglin, and Liu]{moraffah2021causal}
Raha Moraffah, Paras Sheth, Mansooreh Karami, Anchit Bhattacharya, Qianru Wang,
  Anique Tahir, Adrienne Raglin, and Huan Liu.
\newblock Causal inference for time series analysis: Problems, methods and
  evaluation.
\newblock \emph{Knowledge and Information Systems}, pp.\  1--45, 2021.

\bibitem[Morales-Alvarez et~al.(2021)Morales-Alvarez, Lamb, Woodhead, Jones,
  Allamanis, and Zhang]{morales2021vicause}
Pablo Morales-Alvarez, Angus Lamb, Simon Woodhead, Simon~Peyton Jones,
  Miltiadis Allamanis, and Cheng Zhang.
\newblock Vicause: Simultaneous missing value imputation and causal discovery
  with groups.
\newblock \emph{arXiv preprint arXiv:2110.08223}, 2021.

\bibitem[Nauta et~al.(2019)Nauta, Bucur, and Seifert]{nauta2019causal}
Meike Nauta, Doina Bucur, and Christin Seifert.
\newblock Causal discovery with attention-based convolutional neural networks.
\newblock \emph{Machine Learning and Knowledge Extraction}, 1\penalty0
  (1):\penalty0 19, 2019.

\bibitem[Pamfil et~al.(2020)Pamfil, Sriwattanaworachai, Desai, Pilgerstorfer,
  Georgatzis, Beaumont, and Aragam]{pamfil2020dynotears}
Roxana Pamfil, Nisara Sriwattanaworachai, Shaan Desai, Philip Pilgerstorfer,
  Konstantinos Georgatzis, Paul Beaumont, and Bryon Aragam.
\newblock Dynotears: Structure learning from time-series data.
\newblock In \emph{International Conference on Artificial Intelligence and
  Statistics}, pp.\  1595--1605. PMLR, 2020.

\bibitem[Pawlowski et~al.(2020)Pawlowski, Coelho~de Castro, and
  Glocker]{pawlowski2020deep}
Nick Pawlowski, Daniel Coelho~de Castro, and Ben Glocker.
\newblock Deep structural causal models for tractable counterfactual inference.
\newblock \emph{Advances in Neural Information Processing Systems},
  33:\penalty0 857--869, 2020.

\bibitem[Peters et~al.(2012)Peters, Mooij, Janzing, and
  Sch{\"o}lkopf]{peters2012identifiability}
Jonas Peters, Joris Mooij, Dominik Janzing, and Bernhard Sch{\"o}lkopf.
\newblock Identifiability of causal graphs using functional models.
\newblock \emph{arXiv preprint arXiv:1202.3757}, 2012.

\bibitem[Peters et~al.(2013)Peters, Janzing, and
  Sch{\"o}lkopf]{peters2013causal}
Jonas Peters, Dominik Janzing, and Bernhard Sch{\"o}lkopf.
\newblock Causal inference on time series using restricted structural equation
  models.
\newblock \emph{Advances in Neural Information Processing Systems}, 26, 2013.

\bibitem[Peters et~al.(2017)Peters, Janzing, and
  Sch{\"o}lkopf]{peters2017elements}
Jonas Peters, Dominik Janzing, and Bernhard Sch{\"o}lkopf.
\newblock \emph{Elements of causal inference: foundations and learning
  algorithms}.
\newblock The MIT Press, 2017.

\bibitem[Prill et~al.(2010)Prill, Marbach, Saez-Rodriguez, Sorger, Alexopoulos,
  Xue, Clarke, Altan-Bonnet, and Stolovitzky]{prill2010towards}
Robert~J Prill, Daniel Marbach, Julio Saez-Rodriguez, Peter~K Sorger,
  Leonidas~G Alexopoulos, Xiaowei Xue, Neil~D Clarke, Gregoire Altan-Bonnet,
  and Gustavo Stolovitzky.
\newblock Towards a rigorous assessment of systems biology models: the dream3
  challenges.
\newblock \emph{PloS one}, 5\penalty0 (2):\penalty0 e9202, 2010.

\bibitem[Runge(2018)]{runge2018causal}
Jakob Runge.
\newblock Causal network reconstruction from time series: From theoretical
  assumptions to practical estimation.
\newblock \emph{Chaos: An Interdisciplinary Journal of Nonlinear Science},
  28\penalty0 (7):\penalty0 075310, 2018.

\bibitem[Runge(2020)]{runge2020discovering}
Jakob Runge.
\newblock Discovering contemporaneous and lagged causal relations in
  autocorrelated nonlinear time series datasets.
\newblock In \emph{Conference on Uncertainty in Artificial Intelligence}, pp.\
  1388--1397. PMLR, 2020.

\bibitem[Runge et~al.(2019)Runge, Nowack, Kretschmer, Flaxman, and
  Sejdinovic]{runge2019detecting}
Jakob Runge, Peer Nowack, Marlene Kretschmer, Seth Flaxman, and Dino
  Sejdinovic.
\newblock Detecting and quantifying causal associations in large nonlinear time
  series datasets.
\newblock \emph{Science advances}, 5\penalty0 (11):\penalty0 eaau4996, 2019.

\bibitem[Shah \& Peters(2020)Shah and Peters]{shah2020hardness}
Rajen~D Shah and Jonas Peters.
\newblock The hardness of conditional independence testing and the generalised
  covariance measure.
\newblock \emph{The Annals of Statistics}, 48\penalty0 (3):\penalty0
  1514--1538, 2020.

\bibitem[Shimizu et~al.(2006)Shimizu, Hoyer, Hyv{\"a}rinen, Kerminen, and
  Jordan]{shimizu2006linear}
Shohei Shimizu, Patrik~O Hoyer, Aapo Hyv{\"a}rinen, Antti Kerminen, and Michael
  Jordan.
\newblock A linear non-gaussian acyclic model for causal discovery.
\newblock \emph{Journal of Machine Learning Research}, 7\penalty0 (10), 2006.

\bibitem[Shojaie \& Michailidis(2010)Shojaie and
  Michailidis]{shojaie2010discovering}
Ali Shojaie and George Michailidis.
\newblock Discovering graphical granger causality using the truncating lasso
  penalty.
\newblock \emph{Bioinformatics}, 26\penalty0 (18):\penalty0 i517--i523, 2010.

\bibitem[Siggiridou \& Kugiumtzis(2015)Siggiridou and
  Kugiumtzis]{siggiridou2015granger}
Elsa Siggiridou and Dimitris Kugiumtzis.
\newblock Granger causality in multivariate time series using a time-ordered
  restricted vector autoregressive model.
\newblock \emph{IEEE Transactions on Signal Processing}, 64\penalty0
  (7):\penalty0 1759--1773, 2015.

\bibitem[Spirtes et~al.(2000)Spirtes, Glymour, Scheines, and
  Heckerman]{spirtes2000causation}
Peter Spirtes, Clark~N Glymour, Richard Scheines, and David Heckerman.
\newblock \emph{Causation, prediction, and search}.
\newblock MIT press, 2000.

\bibitem[Tank et~al.(2018)Tank, Covert, Foti, Shojaie, and Fox]{tank2018neural}
Alex Tank, Ian Covert, Nicholas Foti, Ali Shojaie, and Emily Fox.
\newblock Neural granger causality for nonlinear time series.
\newblock \emph{stat}, 1050:\penalty0 16, 2018.

\bibitem[Trippe \& Turner(2018)Trippe and Turner]{trippe2018conditional}
Brian~L Trippe and Richard~E Turner.
\newblock Conditional density estimation with bayesian normalising flows.
\newblock \emph{arXiv preprint arXiv:1802.04908}, 2018.

\bibitem[Wu et~al.(2020)Wu, Breuel, Skuhersky, and Kautz]{wu2020discovering}
Tailin Wu, Thomas Breuel, Michael Skuhersky, and Jan Kautz.
\newblock Discovering nonlinear relations with minimum predictive information
  regularization.
\newblock \emph{arXiv preprint arXiv:2001.01885}, 2020.

\bibitem[Xu et~al.(2019)Xu, Huang, and Yoo]{xu2019scalable}
Chenxiao Xu, Hao Huang, and Shinjae Yoo.
\newblock Scalable causal graph learning through a deep neural network.
\newblock In \emph{Proceedings of the 28th ACM international conference on
  information and knowledge management}, pp.\  1853--1862, 2019.

\bibitem[Zhang et~al.(2018)Zhang, B{\"u}tepage, Kjellstr{\"o}m, and
  Mandt]{zhang2018advances}
Cheng Zhang, Judith B{\"u}tepage, Hedvig Kjellstr{\"o}m, and Stephan Mandt.
\newblock Advances in variational inference.
\newblock \emph{IEEE transactions on pattern analysis and machine
  intelligence}, 41\penalty0 (8):\penalty0 2008--2026, 2018.

\bibitem[Zhang \& Hyvarinen(2012)Zhang and Hyvarinen]{zhang2012identifiability}
Kun Zhang and Aapo Hyvarinen.
\newblock On the identifiability of the post-nonlinear causal model.
\newblock \emph{arXiv preprint arXiv:1205.2599}, 2012.

\bibitem[Zhang et~al.(2015)Zhang, Wang, Zhang, and
  Sch{\"o}lkopf]{zhang2015estimation}
Kun Zhang, Zhikun Wang, Jiji Zhang, and Bernhard Sch{\"o}lkopf.
\newblock On estimation of functional causal models: general results and
  application to the post-nonlinear causal model.
\newblock \emph{ACM Transactions on Intelligent Systems and Technology (TIST)},
  7\penalty0 (2):\penalty0 1--22, 2015.

\bibitem[Zheng et~al.(2018)Zheng, Aragam, Ravikumar, and Xing]{zheng2018dags}
Xun Zheng, Bryon Aragam, Pradeep~K Ravikumar, and Eric~P Xing.
\newblock Dags with no tears: Continuous optimization for structure learning.
\newblock \emph{Advances in Neural Information Processing Systems}, 31, 2018.

\end{thebibliography}
\bibliographystyle{iclr2023_conference}

\clearpage
\appendix
\section{ELBO and likelihood derivation}
\label{app: ELBO derivation}
The goal is to derive a lower bound for the joint likelihood $p_\theta(\multitimeX{1},\ldots, \multitimeX{N})$.
\begin{align}
\centering
    &p_\theta(\multitimeX{1},\ldots, \multitimeX{N})\nonumber \\
    =&\log \int p_\theta\left(\multitimeX{1},\ldots, \multitimeX{N}\vert \mG\right)p(\mG)d\mG \nonumber \\
    =&\log \int \frac{\vardist}{\vardist}p_\theta\left(\multitimeX{1},\ldots, \multitimeX{N}\vert \mG\right)p(\mG)d\mG \nonumber \\
    \geq& \int \vardist \log p_\theta\left(\multitimeX{1},\ldots, \multitimeX{N}\vert \mG\right) p(\mG)d\mG + H(\vardist) \label{appeq: ELBO Jensen inequality}\\
    =&\E_{\vardist}\left[\sum_{n=1}^N\log p_{\theta}(\multitimeX{n}\vert\mG)+\log p(\mG)\right] + H(\vardist) \nonumber
\end{align}
where \cref{appeq: ELBO Jensen inequality} is obtained by using Jensen's inequality. 

We can further simplify the likelihood $p_\theta(\multitimeX{n}\vert\mG)$:
\begin{align}
    \log p_\theta(\multitimeX{n}\vert\mG) =& \log \prod_{t=0}^T p_\theta(\mX_t^{(n)}\vert \mX_{<t}^{(n)},\mG)\nonumber\\
    =&\sum_{t=0}^T\log p_\theta\left(\mX_t^{(n)}\vert\mX_{<t}^{(n)},\mG\right)\nonumber \\
    =&\sum_{t=0}^T\sum_{i=1}^D\log p_\theta\left(X_t^{i,(n)}\vert \PaGst, \PaGt\right)\label{appeq: individual likelihood decomposition}
\end{align}
where \cref{appeq: individual likelihood decomposition} is obtained through Markov factorization \citep{lauritzen1996graphical}. 
\section{Structural Identifiability}
\label{app: Structure Identifiability}

In this section, we will focus on proving the structural identifiability of \ModelName{}. Before diving into the details, let us clarify the required assumptions.
\begin{assumption}[Causal Stationarity \citep{runge2018causal}]
The time series process $\mX_t$ with a graph $\mG$ is called \emph{causally stationary} over a time index set $\mathcal{T}$ if and only if for all links $X_{t-\tau}^i\rightarrow X_{t}^j$ in the graph
\[
X_{t-\tau}^i \not \ind X_t^j \vert \mX_{<t}\backslash \{X_{t-\tau}^i\} \text{ holds for all }t\in \mathcal{T}
\]
\label{assumption 1}
\end{assumption}
This characterizes the nature of the time-series data generating mechanism, which validates the choice of the auto-regressive model.
\begin{assumption}[Causal Markov Property \citep{peters2017elements}]
Given a DAG $\mG$ and a joint distribution $p$, this distribution is said to satisfy causal Markov property w.r.t. the DAG $\mG$ if each variable is independent of its non-descendants given its parents.
\label{assumption 2}
\end{assumption}
This is a common assumptions for the distribution induced by an SEM. With this assumption, one can deduce conditional independence between variables from the graph. 
\begin{assumption}[Causal Minimality]
Consider a distribution $p$ and a DAG $\mG$, we say this distribution satisfies causal minimality w.r.t. $\mG$ if it is Markovian w.r.t. $\mG$ but not to any proper subgraph of $\mG$. 
\label{assumption 3}
\end{assumption}
Minimality is also a common assumption for SEMs \citep{hoyer2008nonlinear,zhang2012identifiability, peters2012identifiability}, which can be regarded as a weaker version of \emph{faithfulness} \citep{peters2017elements}. 
\begin{assumption}[Causal Sufficiency]
A set of observed variables $\mV$ is causally sufficient for a process $\mX_t$ if and only if in the process every common cause of any two or more variables in $\mV$ is in $\mV$ or has the same value for all units in the population. 
\label{assumption 4}
\end{assumption}
This assumption implies there are no latent confounders present in the time-series data. 
\begin{assumption}[Well-defined Density]
We assume the joint likelihood induced by the \ModelName{} SEM (\cref{eq: SEM for AR-DECI}) is absolutely continuous w.r.t. a Lebesgue or counting measure and $\vert \log p(\singletimeX;\mG) \vert<\infty$ for all possible $\mG$. 
\label{assumption 5}
\end{assumption}
This assumption is to make sure the induced distribution has a well-defined probability density function. It is also required for the equivalence of the global, local Markov property and Markov factorization property (Theorem 6.22 from \cite{peters2017elements}). 


In the following, we will structure the entire proof into three steps: 
\begin{enumerate}
    \item Prove a general conditions that the \emph{bivariate} time series model needs to satisfy for structural identifiability. This adapts from the theorem 1 in \citet{peters2013causal}.
    \item Prove that a generalized form of SEM, modified from the \emph{post non-linear} (\emph{PNL}) model \citep{zhang2012identifiability}, satisfies the conditions mentioned in step 1. The proposed \ModelName{} (\cref{eq: SEM for AR-DECI}) is a special case of the above SEM.
    \item In the end, we generalize the above indentifiability to the \emph{multivariate} case.
\end{enumerate}

\subsection{General Identifiability Conditions}
\label{subapp: general identifability conditions}
First, we derive the conditions required for identifiability for a general bivariate time series SEM, defined as 
\begin{equation}
    X_t^i = f_{i,t}\left(\PaGst, \PaGt, \epsilon_{t}^i\right).
    \label{appeq: general form of time series SEM}
\end{equation}
We call the above SEM \emph{transition model}, since it only defines the transition behavior rather than the initial conditions. We also need to incorporate a \emph{source model}, which characterizes the initial conditions:
\begin{equation}
    X_s^i = f_{i,s}(\Pa_{G}^{i}, \epsilon_{s}^i)
    \label{appeq: source model SEM}
\end{equation}
for $s\in[0,\sourcelength]$, where $\sourcelength$ is the length for the initial conditions and $\Pa_{G}^i$ contains the parents for node $i$. We define $p_s(\sourcenodes)$ as the induced joint distribution for the initial conditions. 

Now, we prove the following theorem.
\begin{theorem}[Identifiability conditions for bivariate time series]
Assuming \cref{assumption 1}-\ref{assumption 5} are satisfied, given a bivariate temporal process $\mX_{0:T}$ and $\mY_{0:T}$ that are governed by the above SEM (\cref{appeq: general form of time series SEM}) with source model (\cref{appeq: source model SEM}), then the above SEM for the bivariate temporal process is structural identifiable if the following conditions are true:
\begin{enumerate}
    \item Source model $f_{i,s}$ is structural identifiable for all $i=1,\ldots, D$ and $s\in[0,\sourcelength]$.
    \item The transition model (\cref{appeq: general form of time series SEM}) is \emph{bivariate identifiable} w.r.t the \emph{instantaneous parents}. Namely, if graph $\mG$ induced conditional distributions $p(X_t,Y_t|\PaGstXY)$, then $\nexists \mG'$ such that $\mG\neq \mG'$ and the induced conditional $\bar{p}(X_t,Y_t|\bPaGstXY)=p$ for all $t\in[\sourcelength+1, T]$.
\end{enumerate}
where $\PaGstXY$ is the union of the lagged parents of $X_t$ and $Y_t$ under $\mG$, and $\bPaGstXY$ is the union of parents under $\mG'$.
\label{thm: general identifiability}
\end{theorem}
\begin{proof}
We prove this by contradiction. Assume we have an induced joint distribution $p(\singletimeX,\singletimeY)$ under $\mG$, and corresponding $\bar{p}$ under $\mG'$. We further assume the above two conditions in the theorem are met and $p=\bar{p}$ but $\mG\neq\mG'$.

Thus, we have $\KL[p\Vert\bar{p}]$ = 0. Due to the temporal nature of the model, we can further decompose it as the following:
\begin{align*}
    &\KL[p\Vert \bar{p}]\\
    =& \int p(\singletimeX,\singletimeY)\log \frac{p(\singletimeX,\singletimeY)}{\bar{p}(\singletimeX,\singletimeY)}d\singletimeX d\singletimeY\\
    =&\KL[\underbrace{p(\singletimeXsource,\singletimeYsource)}_{p_s}\Vert\underbrace{\bar{p}(\singletimeXsource,\singletimeYsource)}_{\bar{p}_s}]+\int p(\singletimeXsource,\singletimeYsource)\KL[p(\singletimeX[\sourcelength+1],\singletimeY[\sourcelength+1]\vert \singletimeXsource,\singletimeYsource)\Vert \\ &\bar{p}(\singletimeX[\sourcelength+1],\singletimeY[\sourcelength+1]\vert \singletimeXsource,\singletimeYsource)]d\singletimeXsource d\singletimeYsource\\
    =&\KL[p_s\Vert\bar{p}_s]+ \sum_{t=\sourcelength+1}^T\E_{ p(\singletimeXfull{0}{t-1}, \singletimeYfull{0}{t-1})}\left[\KL\left[p(X_t,Y_t\vert \singletimeXfull{0}{t-1}, \singletimeYfull{0}{t-1})\Vert \bar{p}(X_t,Y_t\vert \singletimeXfull{0}{t-1}, \singletimeYfull{0}{t-1})\right]\right]\\
    =&0.
\end{align*}
This means we have $\KL[p_s\Vert \bar{p}_s]=0$ and $\KL\left[p(X_t,Y_t\vert \singletimeXfull{0}{t-1}, \singletimeYfull{0}{t-1})\Vert \bar{p}(X_t,Y_t\vert \singletimeXfull{0}{t-1}, \singletimeYfull{0}{t-1})\right]=0$ almost everywhere. 
Inspired by the strategy used in \citep{peters2013causal}, We consider the following three scenarios:
\paragraph{Disagree on initial conditions} We assume $\mG$ and $\mG'$ disagree on the initial conditions. From the condition 1, we know the source model $f_{i,s}$ is identifiable. Namely, we cannot find $\mG\neq\mG'$ with disagreement on initial conditions such that $\KL[p_s\Vert \bar{p}_s]=0$. This is a contradiction, meaning that $\mG$ and $\mG'$ must agree on the connections between initial set of nodes.
\paragraph{Disagree on lagged parents only} This means for all $t\in[\sourcelength+1,T]$, the instantaneous connections at $t$ for $\mG$ and $\mG'$ are the same, and
$\exists t\in [\sourcelength+1, T]$ such that $\PaGstXY\neq \bPaGstXY$.
We can use a similar argument as the theorem 1 in \citet{peters2013causal}. W.l.o.g., we assume under $\mG$, we have $X_{t-\tau}\rightarrow Y_t$ and there is no connections between them under $\mG'$.
Thus, from Markov conditions, we have 
\[
Y_t\ind X_{t-\tau}\vert \singletimeXfull{0}{t-1}\cup \singletimeYfull{0}{t-1}\cup \text{ND}_{t}^Y\backslash \{Y_t,X_{t-\tau}\}
\]
under $\mG'$,
where $\text{ND}_t^Y$ are the non-descendants of node $Y_t$ at some time $t$. 
However, from the causal minimality and proposition 6.16 in \citet{peters2017elements}, we have 
\[
Y_t\not\ind X_{t-\tau}\vert \singletimeXfull{0}{t-1}\cup \singletimeYfull{0}{t-1}\cup \text{ND}_{t}^Y\backslash \{Y_t,X_{t-\tau}\}
\]
under $\mG$.
This means under this case, $\KL\left[p(X_t,Y_t\vert \singletimeXfull{0}{t-1}, \singletimeYfull{0}{t-1})\Vert \bar{p}(X_t,Y_t\vert \singletimeXfull{0}{t-1}, \singletimeYfull{0}{t-1})\right]\neq0$, which is a contradiction. 

\paragraph{Disagree also on instantaneous parents}
This scenarior means $\exists t\in[\sourcelength+1,T]$ such that they disagree on instantaneous parents. W.l.o.g. we assume $X_t\rightarrow Y_t$ under $\mG$ and $Y_t\rightarrow X_t$ under $\mG'$.

Let's define $\singletimeXfull{0}{t-1}\cup \singletimeYfull{0}{t-1}=\vh$, $\historyGY\subseteq \vh$ contains the values of $\PaGstXY[Y]$ under $\mG$, $\bhistoryGY\subseteq \vh$ contains the parent values under $\mG'$, and $\historyGX$, $\bhistoryGX$ accordingly. 
Thus, the induced conditional distributions from SEM (\cref{appeq: general form of time series SEM}) with $\mG$, $\mG'$ are
\[
p(X_t,Y_t\vert \historyGX\cup\historyGY)\;\;\;\;\text{and}\;\;\;\; \bar{p}(X_t,Y_t\vert \bhistoryGX\cup \bhistoryGY)
\]
From the Markov conditions, we have
\[
p(X_t,Y_t\vert \singletimeXfull{0}{t-1}, \singletimeYfull{0}{t-1}) = p(X_t,Y_t\vert \PaGstXY)
\]
Therefore, we have
\begin{align*}
    &\KL\left[p(X_t,Y_t\vert \vh)\Vert \bar{p}(X_t,Y_t\vert \vh)\right]\\
    =&0\\
    =&\KL[p(X_t,Y_t\vert  \historyGX\cup\historyGY)\Vert \bar{p}(X_t,Y_t\vert \bhistoryGX\cup \bhistoryGY)]\\
\end{align*}
for arbitrary $\vh$,
which contradicts the strucutral identifiability w.r.t. the instantaneous parents. 

In summary, with the two conditions, we cannot find $\mG\neq\mG'$ such that the induced joint $p(\singletimeX,\singletimeY)=\bar{p}(\singletimeX,\singletimeY)$, meaning that the SEMs defined as \cref{appeq: general form of time series SEM} and \cref{appeq: source model SEM} are identifiable w.r.t. \emph{bivariate time series}.
\end{proof}


Since one can use any identifiable static models to characterize the initial behavior of the time series, we will focus on condition 2 for the transition model. In the following, we will show that a generalization of PNL, called \emph{history-dependent PNL}, satisfies condition 2 under assumptions.

\subsection{Identifiability of history-dependent PNL}
\label{subapp: Identifiability of history dependent PNL}
First, we propose a generalization of PNL \citep{zhang2012identifiability} so that it can be history-dependent. For a multivariate temporal process $\singletimeX$, we propose \emph{history-dependent PNL} as 
\begin{equation}
    X_t^i = \nu_{it}\left(f_{it}\left(\PaGst[i], \PaGt[i]\right)+g_{it}\left(\PaGst,\epsilon_{it}\right), \PaGst\right)
    \label{appeq: SEM for history dependent PNL}
\end{equation}
where $\nu_{it}$ is an invertible transformation w.r.t. the first argument. The main differences of the above SEM compared to typical PNL are (1) the invertible transformation $\nu_{it}$ can be history dependent; (2) the inner noise distribution can also be history-dependent.

Next, we show the main theorem about its bivariate identifiability w.r.t. its instantaneous parents. 
\begin{theorem}[History-dependent PNL Bivariate Identifiability]
Assume \cref{assumption 1}-\ref{assumption 5} are satisfied, all transformations in \cref{appeq: SEM for history dependent PNL} and corresponding induced distributions are $3^{rd}$-order differentiable. Given a bivariate temporal process $\singletimeX$, $\singletimeY$, then the history-dependent PNL defined as \cref{appeq: SEM for history dependent PNL} is bivariate identifiable w.r.t its instantaneous parents (i.e.~satisfy condition 2 in \cref{thm: general identifiability}), except for some special cases. 
\label{thm: bivariate identifiable history dependent PNL}
\end{theorem}

\begin{proof}
W.l.o.g. at time $t\in[\sourcelength+1, T]$, we assume $X_t\rightarrow Y_t$ for instantaneous connection under $\mG$ and $Y_t\rightarrow X_t$ under $\mG'$. We fix a value $\vh$ for their entire history $\singletimeXfull{0}{t-1}\cup\singletimeYfull{0}{t-1}=\vh$. With $\vh$, we further define their lagged parents as $\PaGst[X]=\historyGX\subseteq \vh$, $\PaGst[Y]=\historyGY\subseteq \vh$ under $\mG$ and $\bPaGstXY[X]=\bhistoryGX\subseteq\vh$, $\bPaGstXY[Y]=\bhistoryGY$ under $\mG'$.

Therefore, the SEM at time $t$ can be written as
\begin{equation}
    Y_t=\nu\left(f(\historyGY,X_t)+g(\historyGY,\epsilon_{Y}), \historyGY\right)
    \label{appeq in deriv: PNL under G}
\end{equation}
and 
\begin{equation}
    X_t=\bar{\nu}\left(\bar{f}(\bhistoryGX,Y_t)+\bar{g}(\bhistoryGX,\epsilon_{X}), \bhistoryGX\right)
    \label{appeq in deriv: PNL under G'}
\end{equation}
under $\mG$ and $\mG'$, respectively.
Let's assume that their induced conditional distributions at time $t$ are equal (i.e.~violating the identifiable condition (2) in \cref{thm: general identifiability}):
\[
\underbrace{\log p(X_t,Y_t|\historyGX\cup \historyGY)}_{\text{under }\mG} = \underbrace{\log \bar{p}(X_t,Y_t\vert \bhistoryGX\cup \bhistoryGY)}_{\text{under }\mG'}
\]
From the Markov properties, the above equation is equivalent to 
\[
\log p(X_t,Y_t|\vh) = \log \bar{p}(X_t,Y_t|\vh)
\]
with a fixed value $\vh$ of the entire history.

Now, let's define 
\[
\alpha_t = \bnu^{-1}(X_t)\;\;\;\;\; \text{and}\;\;\;\;\; \beta_t=\nu^{-1}(Y_t)
\]
where we omits the dependence of $\bnu^{-1}$ to $\bhistoryGX$ and $\nu^{-1}$ to $\historyGY$. 
It is easy to observe that we have an invertible mapping between $(X_t, Y_t)$ and $(\alpha_t, \beta_t)$. Thus, from the change of variable formula, we have
\[
\log p(X_t,Y_t\vert \vh) = \log p_{\alpha,\beta}(\alpha_t,\beta_t\vert \vh) + \log \vert \mJ \vert
\]
and 
\[
\log \bar{p}(X_t,Y_t\vert \vh) = \log \bar{p}_{\alpha,\beta}(\alpha_t,\beta_t\vert \vh) + \log \vert \mJ \vert
\]
where $\mJ$ is the Jacobian matrix of the transformation. Thus, the equivalence of $\log p$ and $\log \bar{p}$ in the $(X_t, Y_t)$ space can be translated to $(\alpha_t,\beta_t)$ space. 

Thus, from \cref{appeq in deriv: PNL under G}, we have
\begin{equation}
    \beta_t = \Phi(\alpha_t)+g(\historyGY, \epsilon_{Y})
    \label{appeq in deriv: transformed PNL under G}
\end{equation}
under $\mG$. And from \cref{appeq in deriv: PNL under G'}, we have
\begin{equation}
    \alpha_t = \Psi(\beta_t) + \bar{g}(\bhistoryGX,\epsilon_X)
    \label{appeq in deriv: transformed PNL under G'}
\end{equation}
under $\mG'$. This forms an \emph{additive noise model} between $\alphat$, $\betat$ with history-dependent noise. Next, we can use a similar proof techniques as in \citet{hoyer2008nonlinear}.
Here, $\Phi(\cdot) = f(\historyGY,\cdot)\circ \bnu(\bhistoryGX,\cdot)$ and $\Psi(\cdot) = \bar{f}(\bhistoryGX, \cdot) \circ \nu(\historyGY, \cdot)$. 
We further define 
\begin{align*}
    &\eta_1(\alpha_t) = \log p(\alpha_t\vert \vh) &&\bar{\eta_1}(\beta_t) = \log \bar{p}(\beta_t\vert \vh)\\
    &\eta_2(g(\historyGY,\epsilon_{Y})) = \log p_g(g(\historyGY, \epsilon_{Y})|\vh) &&\bar{\eta}_2(\bar{g}(\bhistoryGX, \epsilon_X)) = \log \bar{p}_g(\bar{g}(\bhistoryGX,\epsilon_X)\vert \vh)
\end{align*}
Thus, under $\mG$ (i.e.~\cref{appeq in deriv: transformed PNL under G}), we have
\begin{align}
    \log p(\alphat,\betat\vert \vh) =& \log p(\beta_t\vert \alpha_t, \vh) + \log p(\alpha_t|\vh)\nonumber\\
    =&\eta_2(\beta_t-\Phi(\alpha_t)) + \eta_1(\alpha_t) \label{appeq in deriv: joint likelihood of transformed PNL under G}
\end{align}
Similarly, under $\mG'$ (i.e.~\cref{appeq in deriv: transformed PNL under G'}), we have
\begin{align}
    \log \bar{p}(\alpha_t,\beta_t) = \bar{\eta}_2(\alphat- \Psi(\betat))+ \bar{\eta}_1(\betat)
    \label{appeq in deriv: joint likelihood of transformed PNL under G'}
\end{align}

Based on \cref{appeq in deriv: joint likelihood of transformed PNL under G'}, we have
\[
\frac{\partial^2 \log \bar{p}}{\partial \alphat\partial\betat}=-\bar{\eta}_2''\Psi' \;\;\;\;\text{and}\;\;\;\; \frac{\partial^2 \log \bar{p}}{\partial \alphat^2}=\bar{\eta}_2''
\]
Thus, we have
\[
\frac{\partial}{\partial \alphat}\left(\frac{\partial^2\log \bar{p}/\partial\alphat\partial\betat}{\partial^2\log \bar{p}/\partial\alphat^2}\right) = 0
\]
Due to the equivalence of $\log \bar{p}$ and $\log p$, we apply the above operations to \cref{appeq in deriv: joint likelihood of transformed PNL under G}. After some algebraic manipulation, we obtained the following differential equations for $\eta_2''\Phi'\neq 0$:
\begin{equation}
    \eta_1''' -\frac{\eta_1''\Phi''}{\Phi'} = \left(\frac{\eta_2'\eta_2'''}{\eta_2''}-2\eta_2''\right)\Phi''\Phi' - \frac{\eta_2'''}{\eta_2''}\Phi'\eta_1''+\eta_2'\left(\Phi'''-\frac{(\Phi'')^2}{\Phi'}\right).
    \label{appeq in deriv: PNL differential equation}
\end{equation}
Interestingly, this is exactly equivalent to Eq.(4) in \citet{zhang2012identifiability}. The main difference is the definition of variables and transformations in here are all history-dependent. 

Further, we can also observe that
\[
\betat\ind \bar{g}(\historyGY,\epsilon_Y)\vert \singletimeXfull{0}{t-1}\cup \singletimeYfull{0}{t-1}=\vh.
\]
Since $\betat=\Phi(\alpha_t)+g(\historyGY,\epsilon_Y)$ and $\bar{g}(\bhistoryGX,\epsilon_X)=\alpha_t-\Psi(\betat)$, it is trivial to show the determinant of the Jacobian of the transformation $(\alpha_t,g)$ to $(\betat, \bar{g})$ is $1$. Thus, by a similar argument in theorem 1 from \citet{zhang2012identifiability}, we can derive 
\[
\frac{1}{\Psi'} = \frac{\eta_1''+\eta_2''(\Phi')^2-\eta_2'\Phi''}{\eta_2''\Phi'}
\]
for $\eta_2''\Phi'\neq 0$. 

Thus, the above two differential equations has the same form as theorem 1 in \citet{zhang2012identifiability} where the main difference is that all distributions and transformations involved in our case depend on history $\vh$. 

Therefore, we can directly cite the theorem 8 from \citet{zhang2012identifiability}, which proves that the above differential equations hold true only for $5$ types of special cases. One can refer to Table 1 in \citet{zhang2012identifiability} for details. 
\end{proof}

Corollary 10 from \citet{zhang2012identifiability} validates the choice of using nueral network for the transformation $f$. For completeness, we include it here with slight modification:
\begin{corollary}[Identifiability with neural netowrk $f$]
Assuming the assumptions in \cref{thm: bivariate identifiable history dependent PNL} are true, and the double derivative $(\log p_g(g(\PaGst[Y],\epsilon_Y)\vert \singletimeXfull{0}{t-1}\cup \singletimeYfull{0}{t-1}))''$ w.r.t $\epsilon_Y$ is zero at most at some discrete points. If function $f$ is not invertible \emph{w.r.t. the instantaneous parents}, then, the history-dependent PNL defined as \cref{appeq: SEM for history dependent PNL} is \emph{bivariate identifiable w.r.t. the instantaneous parents} (i.e.~satisfy condition 2 in \cref{thm: general identifiability}). 
\label{corollary: Validity of neural network}
\end{corollary}

It is clear to see that \ModelName{} (\cref{eq: SEM for AR-DECI}) is a special case of the history-dependent PNL (\cref{appeq: SEM for history dependent PNL}), where the outer history-dependent invertible transformation $\nu$ is the identity mapping. Thus, we can directly leverage \cref{thm: general identifiability} together with \cref{thm: bivariate identifiable history dependent PNL} to show \ModelName{} is identifiable w.r.t bivariate time series, and \cref{corollary: Validity of neural network} to validate our design choice (\cref{eq: model design of AR-DECI}).
\subsection{Generalizing to multivariate time series}
\label{appsubsec: generalizing to multivariate time series}
Previously, we prove the identifiability conditions for bivariate time series. In this section, we will generalize it to the multivariate case. 
\begin{theorem}[Generalization to multivariate time series]
Assuming the assumptions in \cref{thm: bivariate identifiable history dependent PNL} are satisfied, we further assume that the multivariate SEM defined in \cref{appeq: SEM for history dependent PNL} satisfies: for each pair of node $i,j\in \mV$,
the SEM 
\[
X_t^i = \nu_{it}\left(f_{it}\left(\PaGst, \PaGt\backslash \{X_t^j\}, \underbrace{\cdot}_{X_t^j}\right)+g_{it}\left(\PaGst, \epsilon_{it}\right), \PaGst\right)
\]
is \emph{bivariate identifiable} w.r.t.~the input, and an identifiable source model is adopted. Then, the history-dependent PNL is \emph{identifiable except for some special cases}. 
\label{thm: multivariate identifiability}
\end{theorem}
\begin{proof}
For this proof, we can follow the strategy used in \cref{thm: general identifiability} and \citet{peters2013causal}. We categorize the difference of the graph $\mG$ and $\mG'$ into three types. Following the same analysis of the $KL$ divergence of the two induced joint distributions, we can see that (1) $\KL[p_s\Vert \bar{p}_s] = 0$ and $\KL[p(\mX_t\vert \singletimeXfull{0}{t-1})\Vert \bar{p}(\mX_t\vert \singletimeXfull{0}{t-1})]=0$.

\paragraph{Disagree on initial conditions} Since we assume that the source model is identifiable, this contradicts $\KL[p_s\Vert \bar{p}_s]=0$. 
\paragraph{Disagree on lagged parents only} We notice that the analysis used in \cref{thm: general identifiability} for this disagreement can be directly translated to multivariate case. The only difference is that the notation $Y_t$, $X_t$ is changed accordingly. 
\paragraph{Disagree also on instantaneous parents}
For this case, with a fixed history value $\vh = \singletimeXfull{0}{t-1}$, the aim is to compare the conditionals $\KL[p(\mX_t\vert \singletimeXfull{0}{t-1}=\vh)\Vert \bar{p}(\mX_t\vert \singletimeXfull{0}{t-1}=\vh)]$. Thus, the problem becomes to how to generalize the bivariate identifiability for instantaneous parents to the multivariate case. We leverage the theorem 2 from \citet{peters2012identifiability}, which proves the multivariate identifiability for any models that belongs to IFMOC. It is easy to see that if the assumptions in \cref{thm: multivariate identifiability} are met, the history-dependent PNL belongs to IFMOC \emph{w.r.t.~the instantaneous parents}. It should be noted that the entire history-dependent PNL \emph{DOES NOT belong to IFMOC}, but this does not affect our results since we only care about the instantaneous parents under this case. 
\end{proof}

\section{Relation to other methods}
\label{appsec: relation to other methods}
\paragraph{VARLiNGaM \citep{hyvarinen2010estimation}} VARLiNGaM \citep{hyvarinen2010estimation} is a causal discovery method for time series data based on the linear vector auto-regression, which can model both lagged and instantaneous effects. Its SEM is defined as \cref{eq: Vector Auto-regressive SEM}, where the noise $\epsilon_{t}^i$ is an independent non-Gaussian noise. It is easy to observe that this is a special case of \ModelName{} (\cref{eq: SEM for AR-DECI}) by setting $f_i$ as the matrix multiplication of the weighted adjacency $\Graph$ with the nodes, and $g_i$ as the identity mapping. For the training objective, VARLiNGaM adopted a two stage training to sidestep the difficulty of directly optimizing the log likelihood. From the \cref{thm: validity of variational objective} for \ModelName{}, we note that the solution from optimizing the variational objective is equivalent to maximizing the log likelihood under infinite data limit. Therefore, by setting large enough DAGness penalty coefficient $\alpha$, $\rho$, the inferred graph from both methods should be equivalent.

\paragraph{DYNOTEARS \citep{pamfil2020dynotears}} The formulation of DYNOTEARS is the same as VARLiNGaM, which is based on linear vector auto-regression. The main novelty is the usage of the DAGness penalty $h(\mG)$, which continuously relaxes the DAG constraint. The training objective is the mean square error with augmented Lagrange scheme for DAGness penalty. Thus, it is obvious that DYNOTEARS is a special case of \ModelName{} with linear transformations and identity $g_i$. Similarly, \cref{thm: validity of variational objective} shows the connections between the variational objective and maximum likelihood, which is equivalent to mean square error if the noise distribution is \emph{Gaussian with equal variances}.

\paragraph{cMLP} cMLP \citep{tank2018neural} combines Granger causality with deep neural networks. The model formulation is 
\[
X_t^i = f_i(\mX_{0:t-1}^1, \ldots, \mX_{0:t-1}^D) + \epsilon_{t}^i
\]
where $f_i$ is a function based on MLP. Although the input is the entire history, the one that matters is the node that has the connection to $X_t^i$ (i.e.~lagged parents). Therefore, it is easy to see they are closely related to \ModelName{} without \emph{instantaneous parents} $\PaGt$ and history-dependent noise. Since the training objective of cMLP is based on the mean square error with sparseness constraint, by the same argument as before, the variational objective is equivalent to mean square error with equal variance Gaussian noise and large training data. 
\paragraph{TiMINo \citep{peters2013causal}} TiMINo is most similar to our work among all the aforementioned methods in terms of model formulation. TiMINo proposed a very general formulation based on IFMOC \citep{peters2012identifiability} and showed the conditions for structural identifiability. \ModelName{} generalizes the TiMINO in a way such that noise history dependency can be incorporated. Thus, \ModelName{} only belongs to IFMOC w.r.t.~the instantaneous parents. Therefore, \ModelName{} without the history-dependent noise is a TiMINo model. The training objective of TiMINo is based on the dependence minimization between the noise residuals and causes, and can only infer summary graph instead of temporal causal graph. \citet{zhang2015estimation} proved the equivalence of the mutual information minimization to maximum likelihood, which is equivalent to our variational objective under infinite data. 
\section{Treatment Effect Estimation}
\label{app: Treatment effect estimation}

We now show how to leverage the fitted \ModelName{} for estimating the \emph{conditional average treatment effect} (CATE). For simplicity, we only consider a special case of CATE defined as 
\begin{equation}
\CATE(a,b) = \E_{\vardist}\left[\E_{p(\mX_{t+\tau}^Y\vert \mX_{<t}, \doop(X_t^I=a), \mG)}[X_{t+\tau}^Y] - \E_{p(\mX_{t+\tau}^Y\vert \mX_{<t}, \doop(X_t^I=b),\mG)}[X_{t+\tau}^Y]\right]
    \label{eq: CATE definition}
\end{equation}
We assume the conditioning variable can only be $\mX_{<t}$ (i.e.~the entire history before $t$), and the intervention and target variable can only be either at current time $t$ or sometime in the future $t+\tau$. We emphasize that this formulation is for simplicity, and \ModelName{} can be easily generalized to more cases as \citet{geffner2022deep}. 
Once fitted, the idea is to draw target samples $X_{t+\tau}^Y$ from the interventional distribution $p(\mX_{t+\tau}^Y\vert \mX_{<t}, \doop(X_t^I), \mG)$ for each graph sample $\mG\sim\vardist$. Then, unbiased Monte Carlo estimation can be used to compute CATE. For sampling from the interventional distribution, we can use the "multilated" graph $\mG_{\doop(X_t^I)}$ to replace $\mG$, where all incoming edges to $X_t^I$ are removed. The intervention samples can be obtained by simulating the \ModelName{} with history $\mX_{<t}$, $X_t^I=a \text{ or } b$ and $\mG_{\doop(X_t^I)}$.

\subsection{Causal Inference Results}
\label{sec:app_syn_results}

\begin{figure}[!htb]
    \centering
    \includegraphics[width=\linewidth]{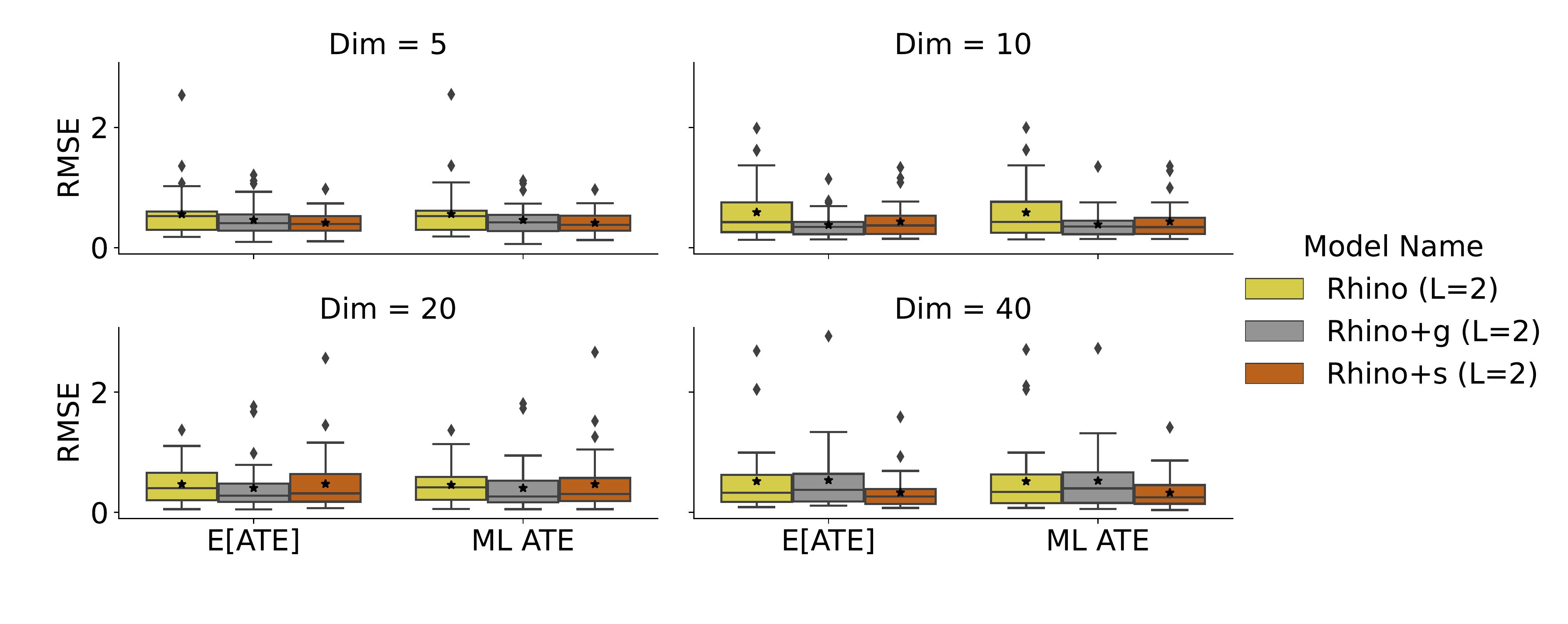}
    \caption{Comparison of the RMSE of the average treatment effects (CATEs) of the different instantiations of \ModelName{} depending on the dimensionality. $\mathbb{E}[\mathrm{CATE}]$ refers to RMSE of the expected CATE over the posterior graph distribution (i.e.~$\mG\sim q_\phi(\mG)$). ML ATE uses the most likely graph to calculate the ATE. These results are obtained by averaging 160 datasets, similar to the discovery setup.}
    \label{fig:synth_ate}
\end{figure}

Here, we provide the preliminary results for CATE performance of \ModelName{} by calculating the RMSEs of the estimated CATEs comparing to the true CATE from the interventional samples (lower is better). We present boxplots of the performance in \cref{fig:synth_ate}. All \ModelName{}-based method perform similarly. Surprisingly, the CATE performance seems to have little correlation to the causal discovery performance and warrants further study in the future.
\section{Variational distribution formulation}
\label{app: variational dist formulation}
Here we provide the detailed formulation of the independent Bernoulli distribution $\vardist$. Since this distribution is responsible for modelling the temporal adjacency matrix $\mG_{0:K}$, we use $\Sigma_k$ to represents the edge probability in $\mG_k$. We further split the edge probability matrices into the instantaneous part $\Sigma_0$ and lagged parts $\Sigma_{1:K}$. 

To avoid the constrained optimization of $\Sigma_{1:K}$ (i.e.~the value needs to be within $[0,1]$), we adopt the following formulation:
\begin{equation}
    \sigma_{k,ij} = \frac{\exp(u_{k,ij})}{\exp(u_{k,ij})+\exp(v_{k,ij})} 
\end{equation}
where $u_{k,ij}\in \mU_k$, $v_{k,ij}\in \mV_k$ and $\mU_k, \mV_k\in\mathbb{R}^{D\times D}$ for all $k=1,\ldots,K$. Since we do not require lagged adjacency matrix to be a DAG, $\mU_k, \mV_k$ has no constraints during optimization. 

On the other hand, $\mG_0$ needs to be a DAG for instantaneous effect. By smart formulation, we can get rid of the length-1 cycles. The intuition is that for a pair of node $i,j$, only three mutually exclusive possibilities can exist: (1) $i\rightarrow j$; (2) $j\rightarrow i$; (3) no edge between them. Thus, instead of using a full probability matrix $\Sigma_0$, we use three lower triangular matrices $\mU_0$, $\mV_0$ and $\mE_0$ to characterise the above three scenarios. For node $i>j$, 
\begin{align*}
p(i\rightarrow j ) &= \frac{\exp(u_{ij})}{\exp(u_{ij})+\exp(v_{ij})+\exp(e_{ij})}\\
p(j\rightarrow i) &= \frac{\exp(v_{ij})}{\exp(u_{ij})+\exp(v_{ij})+\exp(e_{ij})}\\
p(\text{no edge}) &= \frac{\exp(e_{ij})}{\exp(u_{ij})+\exp(v_{ij})+\exp(e_{ij})}.\\
\end{align*}
Thus, by this formulation, the corresponding instantaneous adjacency matrix will not contain length-1 cycles. 
\section{Synthetic Experiments}
\subsection{Data generation}
\label{sec:app_syn_data}

We create the synthetic datasets in a four step process: 1) generate random Erd\"{o}s–R\'enyi (ER) or scale-free (SF) graphs that specify the lagged and instantaneous causal relationships; 2) drawing random MLPs for the functional relationships as well as a random \emph{conditional} spline transformation to modulate the scale of the Gaussian noise variables $\epsilon$; 3) sample initial starting conditions and follow \cref{eq: Temporal SEM} with the additive noise to simulate the temporal progression; 4) removing the burn-in period and return stable timeseries. We consider four different axes of variation for the data generation: number of nodes $N_{nodes} \in [5, 10, 20, 40]$; ER or SF graphs; instantaneous or no instantaneous effects; and history-dependent or history-independent noise (i.e.~Gaussian noise). All combinations are generated with 5 different seeds, yielding 160 different datasets. 
Datasets with instantaneous effects have $4 \times N_{nodes}$ edges in the instantaneous adjacency matrix. All datasets have $2 \times N_{nodes}$ connections in the lagged adjacency matrices.
The MLPs for the functional relationships are fully-connected with two hidden layers,64 units and ReLU activation. In case of history-independent noise, we are using Gaussian as the base distribution. The history dependency is modelled as a product of a scale variable obtained by the transformation of the averaged lagged parental values through a random-sampled quadratic spline, and Gaussian noise variable.

The datasets with 40 nodes are generated with a series length of 400 steps, a burn-in period of 100 steps, and 100 training series. All other datasets are generated with a time-series length of 200, burn-in period of 50 steps and 50 training series. We generate random interventions for all the datasets by setting the treatment variable to 10 for intervention and -10 for reference. 5000 ground-truth intervention samples are used to estimate the true treatment effect.

\subsection{Methods}
\label{sec:app_syn_methods}
All benchmarks for the synthetic experiments are run by using publicly available libraries: VARLiNGaM \cite{hyvarinen2010estimation} is implemenented in the \texttt{lingam}\footnote{see \url{https://lingam.readthedocs.io}} python package. PCMCI$^{+}$\citep{runge2020discovering} is implemented in \texttt{Tigramite}\footnote{see \url{https://jakobrunge.github.io/tigramite/}}. We use the implementation in \texttt{causalnex}\footnote{see \url{https://causalnex.readthedocs.io/en/latest/}} to run DYNOTEARS\citep{pamfil2020dynotears}. We use the default parameters for all these baselines. For PCMCI$^{+}$, we enumerate all graphs in the Markov equivalence class to evaluate the causal discovery performance (see \cref{subsubapp: DREAM3 adj matrix aggregation} for details).

For \ModelName{} and its variants, we use the same set of hyper-parameters for all 160 datasets to demonstrates our robustness. By default, we allow \ModelName{} and its variants to model instantaneous effect; set the model lag to be the ground truth $2$ except for ablation study; the $\vardist$ is initialized to favour sparse graphs (edge probability$<0.5$); quadratic spline flow is used to for history-dependent noise. For the model formulation, we use 2 layer fully connected MLPs with 64 (5 and 10 nodes),  80 (10 nodes) and 160 (40 nodes) for all neural networks in \ModelName{}-based methods. We also apply layer normalization and residual connections to each layer of the MLPs. For the gradient
estimator, we use the Gumbel softmax method with a hard forward pass and a soft backward pass
with temperature of 0.25. All spline flows uses 8 bins. The embedding sizes for transformation (i.e.~\cref{eq: model design of AR-DECI} and conditional spline flow) is equal to the node number. 

For the sparseness penalty $\lambda_s$ in \cref{eq: AR-DECI Graph Prior}, we use 9 for \ModelName{} and \ModelName{}+s, and 5 for \ModelName{}+g. We set $\rho=1$ and $\alpha=0$ for all \ModelName{}-based methods. For optimization, we use Adam \citep{kingma2014adam} with learning rate 0.01. The training procedure follows from Appendix B.1 in \citet{geffner2022deep}. 


\subsection{Additional Causal Discovery Results}
\label{sec:app_syn_results}

\begin{figure}[!htb]
    \centering
    \includegraphics[width=\linewidth]{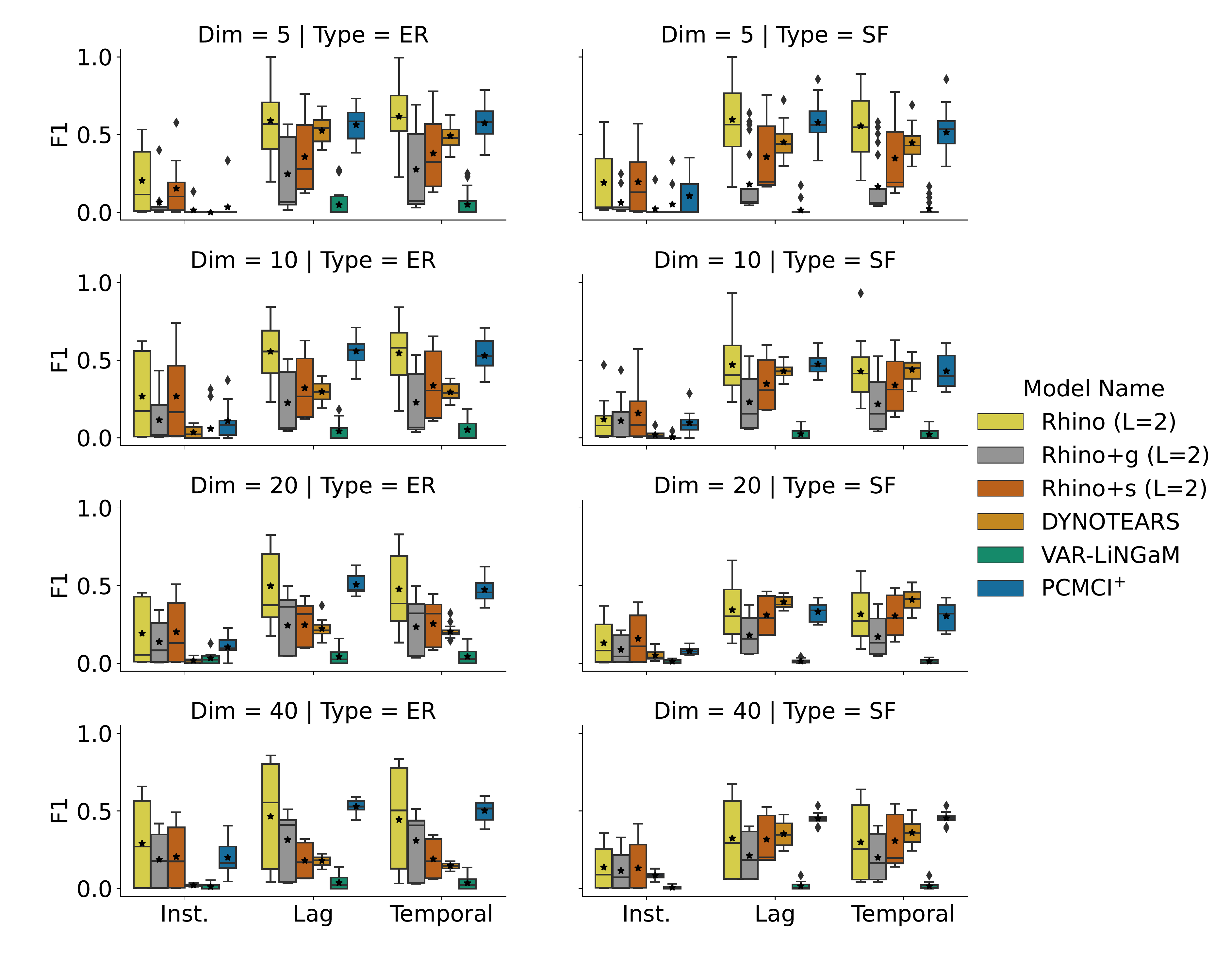}
    \caption{Comparison of the F$_1$ score of the different baseline methods as well as \ModelName{} (light yellow) depending on the dimensionality and the graph type. Inst. refers to the performance on the instantaneous adjacency matrix, Lag refers to the lagged adjancency matrices and temporal considers the full temporal matrix.}
    \label{fig:synth_er_sf}
\end{figure}

\paragraph{Ablation: different type of graphs} The first study is to test our model robustness to different types of graphs. \cref{fig:synth_er_sf} shows the discovery performance over ER or SF graph averaged over all other possible data setting combinations. Most methods perform better on ER graphs than on SF graphs, with only DYNOTEARS \citep{pamfil2020dynotears} as an exception. We note that the PCMCI$^{+}$ runs on SF graphs with 40 nodes exceed our maximum run time of \textbf{1 week}, showing its computational limitation in high dimensions. Nevertheless, \ModelName{} achieves consistent performance throughout all graph settings.

\begin{figure}[!htb]
    \centering
    \includegraphics[width=\linewidth]{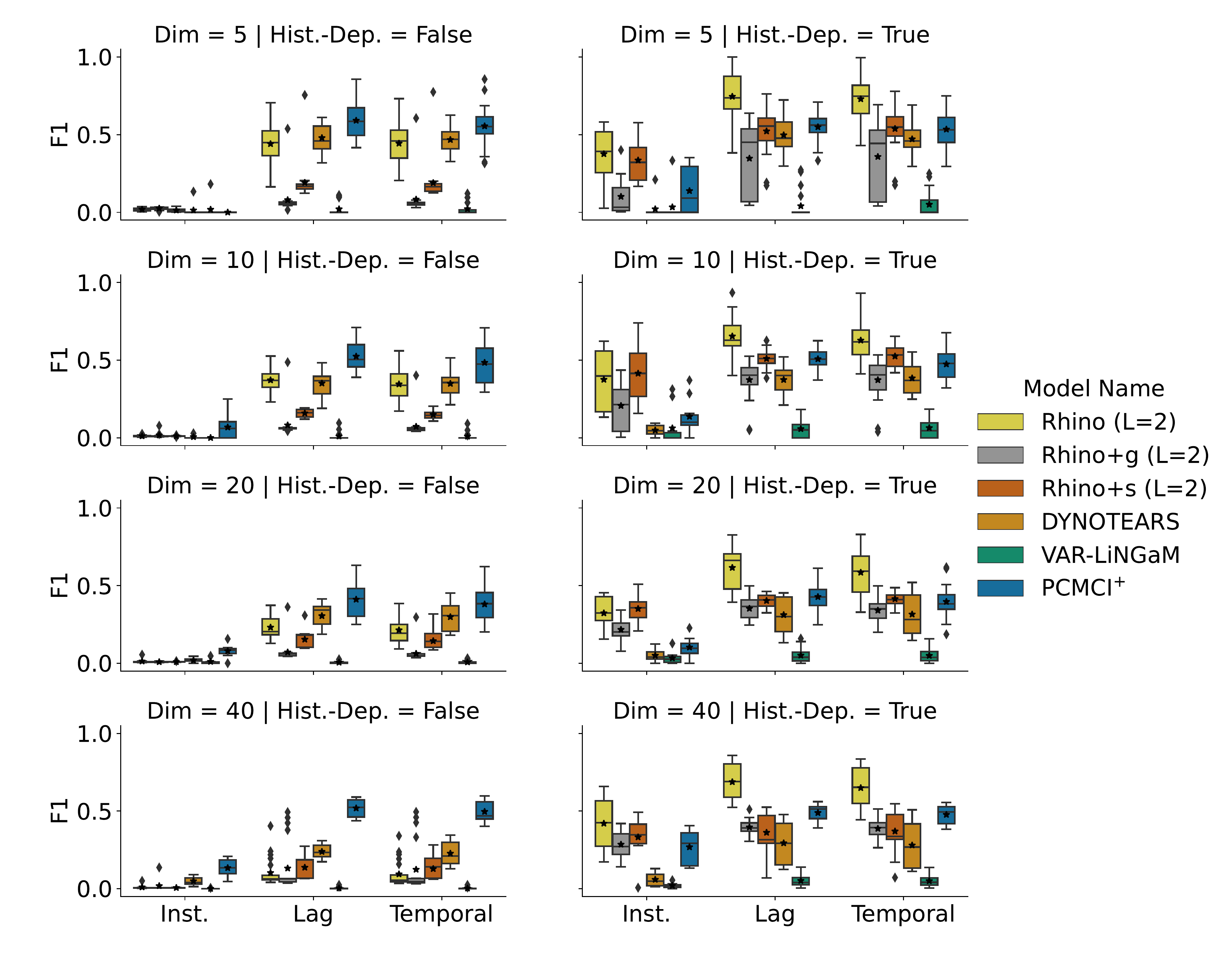}
    \caption{Comparison of the F$_1$ score of the different baseline methods as well as \ModelName{} (light yellow) depending on the dimensionality and whether the data is generated with history-depence or not. Inst. refers to the performance on the instantaneous adjacency matrix, Lag refers to the lagged adjancency matrices and temporal considers the full temporal matrix.}
    \label{fig:synth_histdep}
\end{figure}

\paragraph{Ablation: history dependency} \Cref{fig:synth_histdep} explores the performance difference of all methods on data generated with/without history-dependent noise. Interestingly, most methods perform better on the history-dependent datasets than the history-independent ones. The possible reasons are (1) the difficulty of the discovery also depends on the randomly sampled functions; (2) the default hyperparameters of all methods are initially chosen to favor the datasets with history-dependent noise and instantaneous effects. 
We find that PCMCI$^{+}$ is the most robust across both settings, followed by \ModelName{} and DYNOTEARS. On the other hand, the two variants of \ModelName{} seems to be less robust. When the \ModelName{} is correctly specified, it achieves the best performance. In summary, \ModelName{} demonstrates reasonable robustness to history-dependency mismatch and achieves the best when correctly specified. 

\begin{figure}[!htb]
    \centering
    \includegraphics[width=\linewidth]{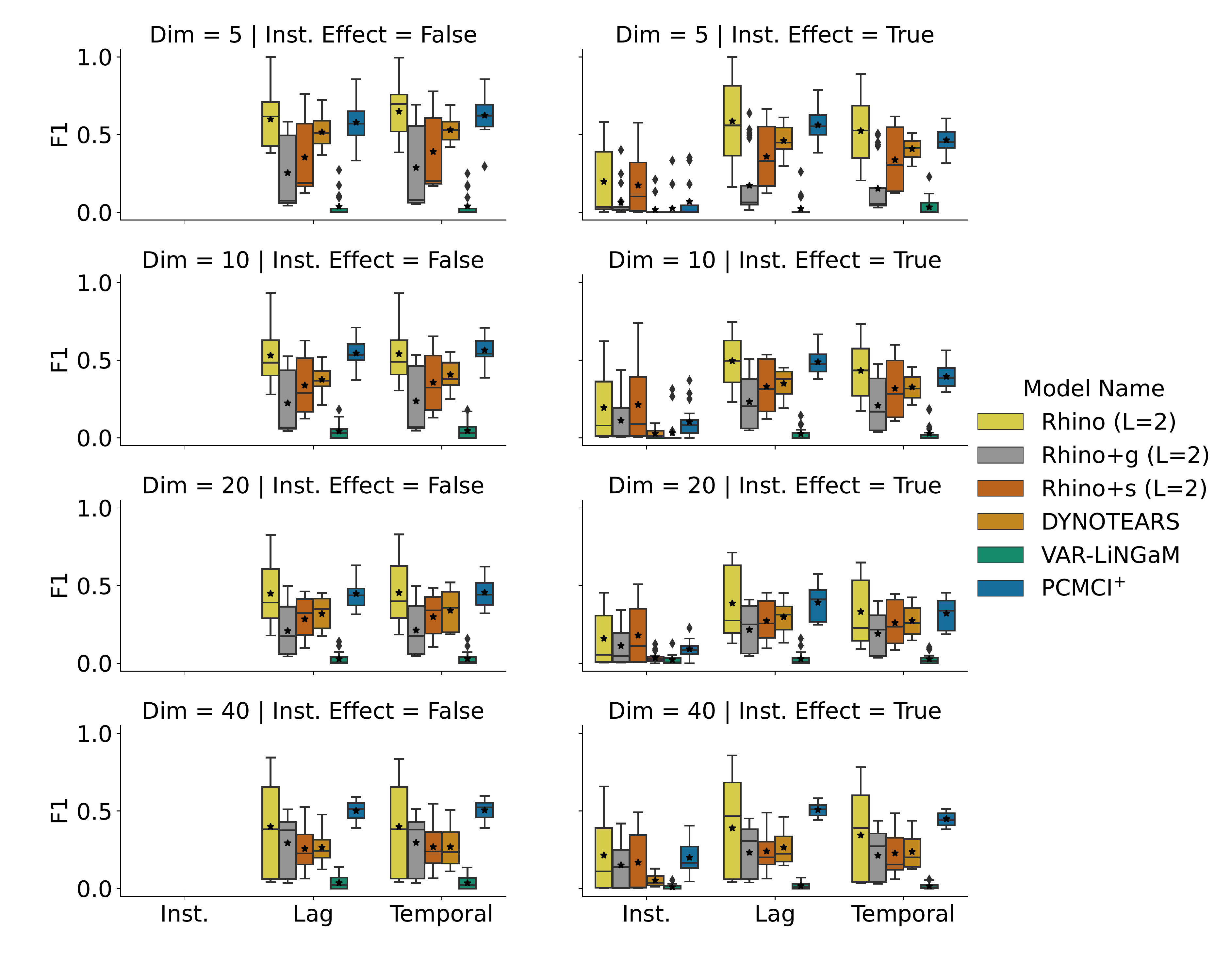}
    \caption{Comparison of the F$_1$ score of the different baseline methods as well as \ModelName{} (light yellow) depending on the dimensionality and whether the data is generated with instantaneous effects or not. Inst. refers to the performance on the instantaneous adjacency matrix, Lag refers to the lagged adjancency matrices and temporal considers the full temporal matrix.}
    \label{fig:synth_inst}
\end{figure}

\paragraph{Ablation: instantaneous effect} We investigate the impact of instantaneous effects in the data. \Cref{fig:synth_inst} shows the F$_1$ score averaged over all possible setting combinations other than instantaneous effect. 
All methods seem to be robust across both settings with PCMCI$^{+}$ and \ModelName{} performing the best. The score of the instantaneous adjacency matrix when instantaneous effects are disabled is not defined and therefore not plotted.

\section{Real-world Experiment Details}
\label{app:Experiment details}

\subsection{DREAM3 Hyperparameter setting}
\label{subsubapp: DREAM3 Hyperparams}
For tuning the hyper-parameters of \ModelName{}, its variants and DYNOTEARS, we split each of the 5 datasets into $80\%/20\%$ training/validation. We tune \ModelName{} and its variants based on the validation likelihoods, and DYNOTEARS based on the validation RMSE error. For PCMCI$^+$, we use the default settings recommended in the Tigramite package (\url{https://github.com/jakobrunge/tigramite}). For other Granger causality baselines, refer to Table 7-11 in \citet{khanna2019economy}. 

\begin{table}[!htb]
\begin{tabular}{lllllll}
\toprule
Hyperparams                   & Node Embedding & Instantaneous eff. & Node Embed. (flow) & lag & $\lambda_s$ & Auglag \\ \midrule
\ModelName{} (Ecoli1) &16                & False                   & 16                   & 2    &  19 & 30         \\ 
\ModelName{} (Ecoli2) &16                & False                   & 100                   & 2    &  25 & 80      \\
\ModelName{} (Yeast1) &32               & False                   & 100                   & 2    &  25  & 10        \\ 
\ModelName{} (Yeast2) &32               & False                   & 100                   & 2    &  25  & 80        \\ 
\ModelName{} (Yeast3) &32                & False                   & 16                   & 2    &  25  & 5        \\

\ModelName{}+g (Ecoli1) &100                & False                   & N/A                   & 2    &  15 & 60         \\ 
\ModelName{}+g (Ecoli2) &100                & False                   & N/A                    & 2    &  25 & 25      \\
\ModelName{}+g (Yeast1) &100               & False                   & N/A                    & 2    &  15  & 5        \\ 
\ModelName{}+g (Yeast2) &100               & False                   & N/A                    & 2    &  19  & 125       \\ 
\ModelName{}+g (Yeast3) &100                & False                   & N/A                    & 2    &  9  & 10        \\

\bottomrule
\end{tabular}
\caption{The hyperparameter setup for \ModelName{}. \texttt{Node embedding} is the dimensionality of $\vu_{\tau,i}$ below \cref{eq: model design of AR-DECI}; \texttt{Instantaneous eff.} specifies whether it models the instantaneous effect or not; \texttt{Node Embed. (flow)} represents the dimensionality of the node embedding for the hyper-network used for conditional spline flow $g_i$ since the hyper-network shares the similar structure as \cref{eq: model design of AR-DECI}; \texttt{lag} defines the model lag order; and $\lambda_s$ is the sparseness penalty in the prior (\cref{eq: AR-DECI Graph Prior}); \texttt{Auglag} is the number of augmented Lagrangian steps, each step consists of 2000 training iterations.}
\label{tab: DREAM3 Rhino hyperparams}
\end{table}

Other than the hyper-parameters reported in \cref{tab: DREAM3 Rhino hyperparams}, we use 1-layer MLPs with 10 hidden units for both $\ell_{\tau,j}, \zeta_{i}$ in \cref{eq: model design of AR-DECI} and the hyper-network for conditional spline flow (8 bins). All the MLPs use residual connections and layer-norm at every hidden layer. We use linear conditional spline flow \citep{dolatabadi2020invertible} instead of the original quadratic version \citep{durkan2019neural} for better training stability. We also initialise the Bernoulli probability $\vardist$ to favour dense graphs (i.e. edge probability $>$ 0.5). For prior $p(\mG)$, we set the initial value $\rho=1$ and $\alpha=0$. For the gradient estimator, we use the Gumbel softmax
method with a hard forward pass and a soft backward pass with temperature of 0.25. We use batch size 64, learning rate 0.001 with Adam optimizer \citep{kingma2014adam}. The training procedure follows from Appendix B.1 in \citet{geffner2022deep}. 

\begin{table}[!htb]
\centering
\begin{tabular}{llll}
\toprule
Hyperparams & lag & $\lambda_a$ & $\lambda_w$ \\ \midrule
Ecoli1      & 2   & 0.01      & 0.5       \\
Ecoli2      & 2   & 0.1       & 0.01      \\
Yeast1      & 2   & 0.005     & 0.1       \\
Yeast2      & 3   & 0.01      & 0.01      \\
Yeast3      & 2   & 0.01      & 0.005     \\ \bottomrule
\end{tabular}
\caption{The hyperparameter setup for DYNOTEARS.}
\label{tab: DREAM3 hyperparams dynotears}
\end{table}
\cref{tab: DREAM3 hyperparams dynotears} contains the hyper-parameters setup for DYNOTEARS. We set the maximum training iterations to be 1000 with DAGness tolerance $10^{-8}$. The threshold value for the weighted adjacency matrix is $0.05$. For PCMCI$^+$, the maximum lag is set to 2. The conditional independence test is set to \texttt{parcorr}, which is based on linear ordinary least square (OLS). A more powerful choice can be a nonlinear independence test based on GP, called \texttt{GPDC}. However, PCMCI$^+$ with $GPDC$ is too slow to finish the training. 

\subsection{Post-processing temporal adjacency matrix}
\label{subsubapp: DREAM3 adj matrix aggregation}
The ground truth graphs for DREAM3 and Netsim datasets are summary graph, which is essentially the temporal graph aggregated over time. We provide a formal definition of summary graph:
\begin{definition}[Causal summary graph
\label{def: summary graph}
\citep{assaad2022survey}]
Let $\mX_t$ be a multivariate temporal process, and $\mG=(\mV,\mE)$ be a summary graph. The edge $p\rightarrow q$ exists if and only if there exists some time $t$ and some lag $\tau$ such that $\mX_{t-\tau}^p$ causes $\mX_{t}^q$ at time $t$ with a lag $0\leq i$ for $p\neq q$ and with a time lag of $0<i$ for $p=q$. 
\end{definition}

Unlike the some of the Granger causality baselines, \ModelName{} (and its variants), DYNOTEARS, VARLiNGaM produces the temporal adjacency matrix after training. For DREAM3 and Netsim datasets, this creates the incompatibility during evaluation. Thus, we need to aggregate the temporal graph into a summary graph before comparing to the ground truth. For binary adjacency matrix, we sum over the time steps followed by a step function, i.e.~$\text{step}(\sum_{k}\mG_k)$. Thus, there will be an edge $i\rightarrow j$ in summary graph as long as there is a connection from $i$ to $j$ at any timestamp. For the Bernoulli probability matrix from \ModelName{} and its variants, we take a $\max(\cdot)$ over the timestamp to generate the probability matrix for the summary graph. 

An exception is PCMCI$^+$, which can only produce MECs for the instantaneous adjacency matrix. In such case, we will enumerate up to 10000 possible instantaneous DAGs from the MECs. Together with the lagged adjacency matrix, we will perform the above post-processing step to generate the corresponding aggregated adjacency matrix. We also estimate the corresponding edge probabilities by taking the average over all possible DAGs. 

For DREAM3 experiments, we ignore the self-connections by setting the diagonal of the aggregated adjacency matrix to be 0.

For Netsim, self-connections are not ignored, following the same settings as \citet{khanna2019economy}.

\subsection{Additional DREAM3 Results}
\label{subsubapp: additional DREAM3 ROC}
Here, \cref{fig:DREAM3 Roc curve plot} shows the additional ROC curve plots for all 5 datasets in DREAM3. For the visualization purpose, we only select a single run for \ModelName{} and this will not affect the curve much due to small standard error in \cref{table: Exp DREAM3 AUROC}.
\begin{figure}
    \centering
    \includegraphics[scale=0.36]{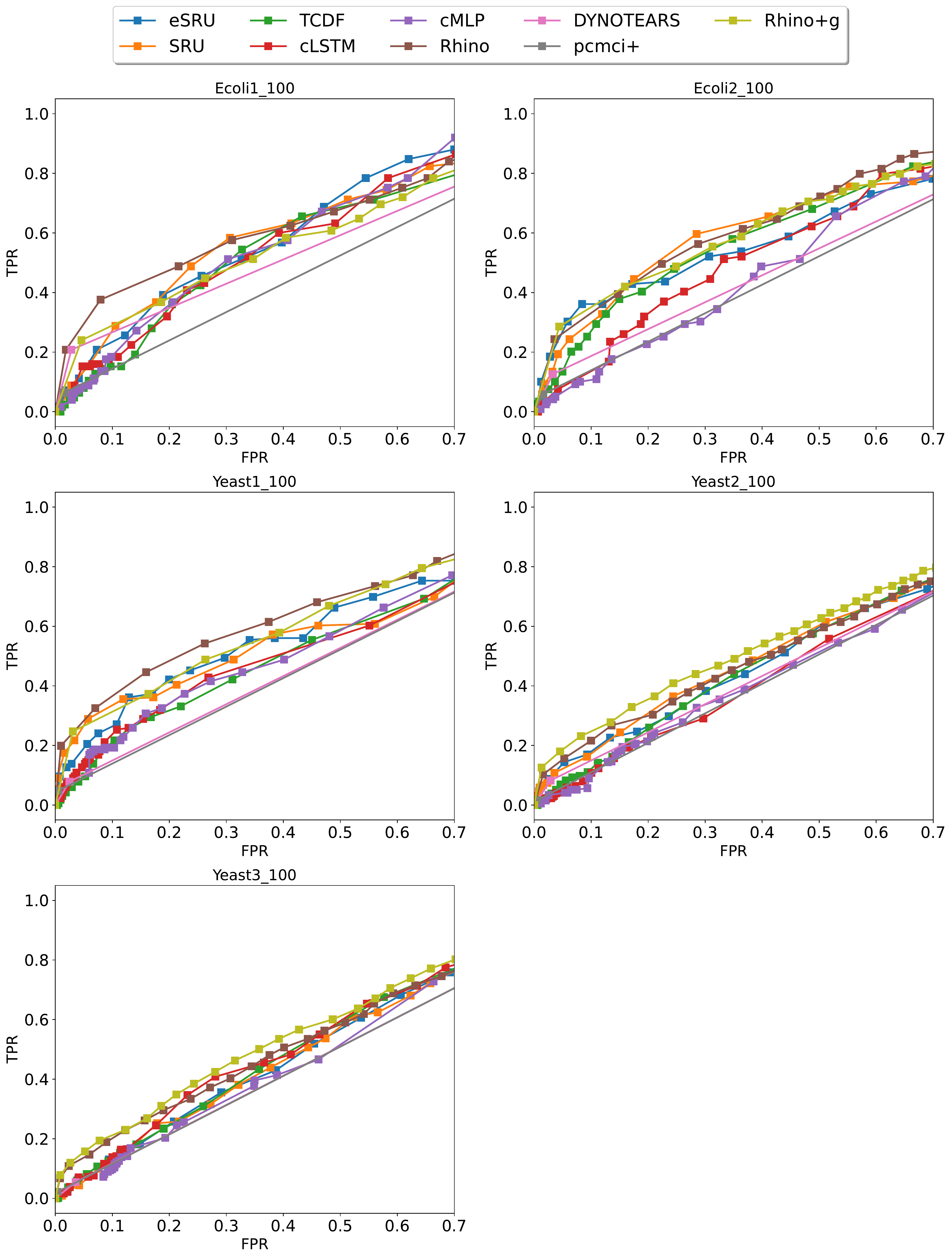}
    \caption{The ROC curve plots of \ModelName{} and other baselines for DREAM3 datasets. For illustration purpose, we only select a single run of \ModelName{}, \ModelName{}+g, DYNOTEARS and PCMCI$^+$ to plot ROC curve. Since the standard error reported in \cref{table: Exp DREAM3 AUROC} is relatively small, the plot should not vary much for other runs. The ROC curve of other baselines are directly taken from figure 7 in \citet{khanna2019economy}. }
    \label{fig:DREAM3 Roc curve plot}
\end{figure}

\subsection{Netsim Hyperparameter setting}
\label{subsubapp: Netsim hyperparams}
For the Netsim experiment, we extract subject 2-6 in \emph{Sim-3.mat} to form the training data and use subject 7-8 as validation dataset. Following the same settings as DREAM3 (\cref{subsubapp: DREAM3 Hyperparams}), we tune the hyperparameters of \ModelName{} and its variants based on the validation log likelihood; DYNOTEARS with MSE on validation dataset; and use default settings of PCMCI$^+$ from Tigramite package. 

It is worth noting that unlike DREAM3 experiment, where the results and hyperparameters of Granger causality baselines can be directly taken from \citet{khanna2019economy}. Their setup of Netsim experiment is different from ours, where they train the baselines using a \textbf{single subject} and compute the corresponding AUROC, followed by averaging over subjects 2-6. Our setup is to train all methods using the entire data from subject 2-6 before computing AUROC. Thus, the hyperparameters for Granger causality are slightly different, and the AUROC increases for the baselines compared to those reported in \citet{khanna2019economy}. 

\paragraph{\ModelName{}} The hyperparameters are the same as DREAM3, except for the following: we initialise the Bernoulli probability of $\vardist$ to have no preference (i.e.~edge probability$=0.5$); the $\lambda_s=25$; we use 2 layer MLPs with 64 hidden units for both functional model (\cref{eq: model design of AR-DECI}) and hyper-network with embedding size 15; the augmented Lagrangian step is 5. For \ModelName{} variants, we use the above settings as well.

\paragraph{DYNOTEARS, PCMCI$^+$ and VARLiNGaM} For DYNOTEARS, we set lag to be 2, $\lambda_a=0.5$ and $\lambda_w=0.5$. For PCMCI$^+$, we use \texttt{parcorr} independence test with lag 3. For VARLiNGaM, we use lag 2 with default settings as \url{https://lingam.readthedocs.io/en/latest/}. 

\paragraph{Granger Causality} For computing AUROC, we follow the same method as \citet{khanna2019economy, tank2018neural} by sweeping through a range of hyperparameters. Specifically, we use the same hyperparameters for SRU and eSRU as \citep{khanna2019economy}. For cMLP, we choose the ridge penalty as $0.43$ and sweep through the group sparse penalty in range $[0.1,1]$. For cLSTM, we set the ridge penalty to be 0.045, and sweep the group sparse penalty in range $[0.1,1]$.For TCDF, we sweep through the threshold in range $[-1,2]$ for the attention scores. Other than the above hyperparameters, everything else follows the setup as in \citet{khanna2019economy}.

\subsection{Additional Netsim Results}
\label{subsubapp: additional Netsim ROC}
\Cref{fig:Netsim ROC plot} shows the ROC curve plot for \ModelName{} and other baselines. It is clear that \ModelName{} achieves significantly better TPR-FPR trade-offs compared to others. 

\begin{figure}
    \centering
    \includegraphics[scale=0.4]{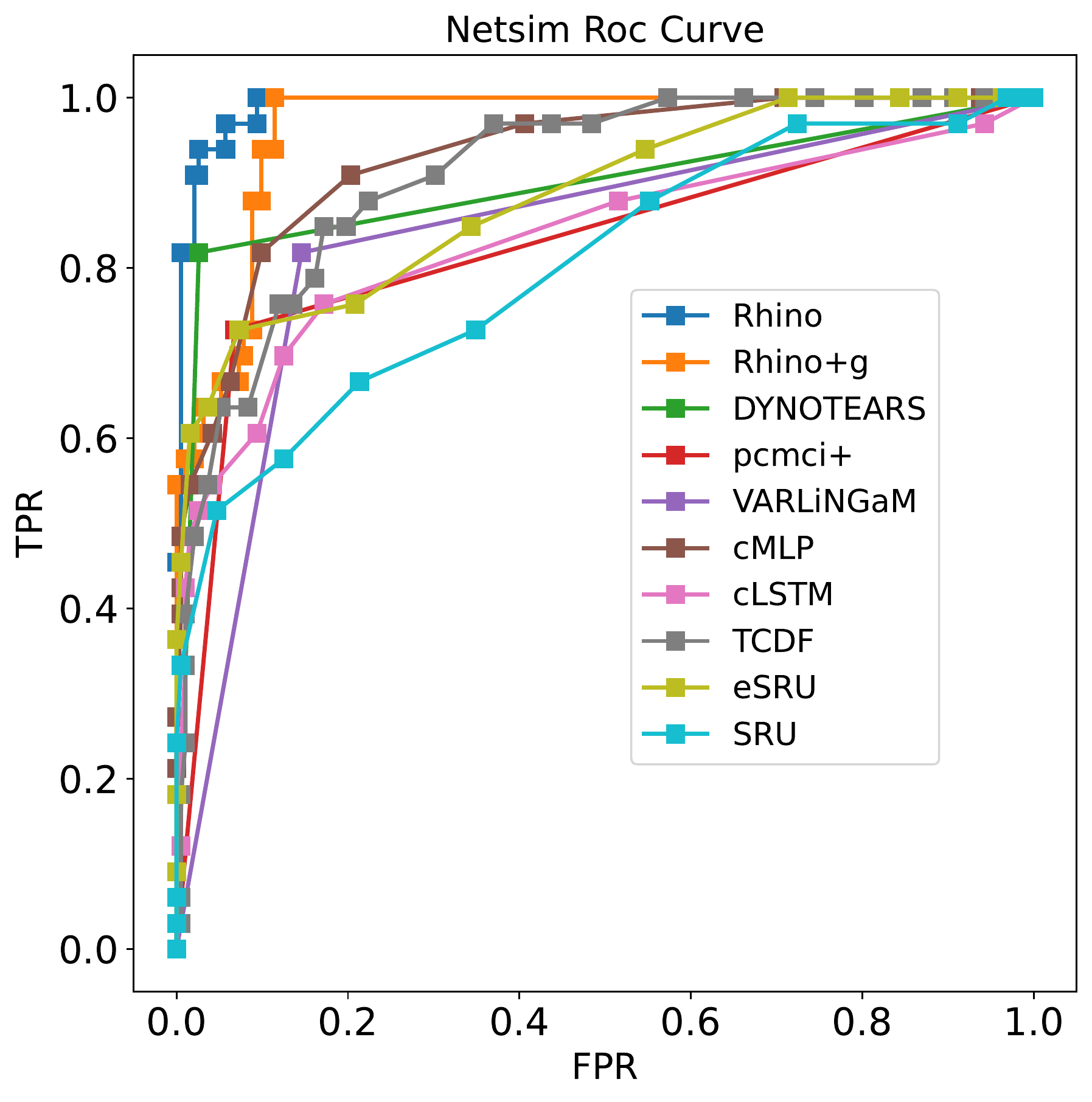}
    \caption{The ROC curve plots of \ModelName{} and other baselines for Netsim dataset. Similar to \cref{fig:DREAM3 Roc curve plot}, we only select 1 run out of 5 for \ModelName{}, \ModelName{}+g, DYNOTEARS, PCMCI$^+$ for illustration purpose. }
    \label{fig:Netsim ROC plot}
\end{figure}

\section{AUROC Metric}
\label{app: netsim AUROC metric}
AUROC metric is a one of the standard metrics for evaluating the causal discovery, which measures the trade-off between the \emph{true positive rate} (TPR) and \emph{false positive rate} (FPR). However, during the experiments, we found out that AUROC does not necessarily correlate well with other discovery metrics. From \cref{fig:Netsim AUROC F1 valid curve}, it is clear that the F$_1$ score continues to increase whereas AUROC and validation likelihood starts to decrease after few steps. Since the dataset of Netsim is relatively small, this indicates the possible overfitting. This disagreement originates from the different aspects these metrics care about. For AUROC, it cares about the trade-off between TPR and FPR with various decision thresholds, and it penalizes the wrong decisions with certainty harshly. On the other hand, F$_1$ score cares about the final inferred binary adjacency matrix with a fixed decision threshold. For example, if we multiply the Bernoulli probability matrix by a small factor (e.g.~$10^{-5}$), the AUROC score will remain the same but the F$_1$ score will tends to 0 with the  default decision threshold $0.5$.

Thus, model overfitting tends to drive the edge probabilities towards $1$ or $0$, which may help the F$_1$ score but these extreme decisions can result in a large decrease in the AUROC score. Thus, for small dataset, we believe AUROC is a better metric than F$_1$, which also agrees with validation likelihood.

In addition, the Bayesian setup of \ModelName{} may also help with better AUROC for small dataset. From the same figure, even the large decrease of validation likelihood suggests potential model overfitting, the AUROC still maintains a reasonable value. This may be due to the Bayesian view of the causal graph, where the posterior edge probability does not converge to extreme values.
\begin{figure}
    \centering
    \includegraphics[scale=0.4]{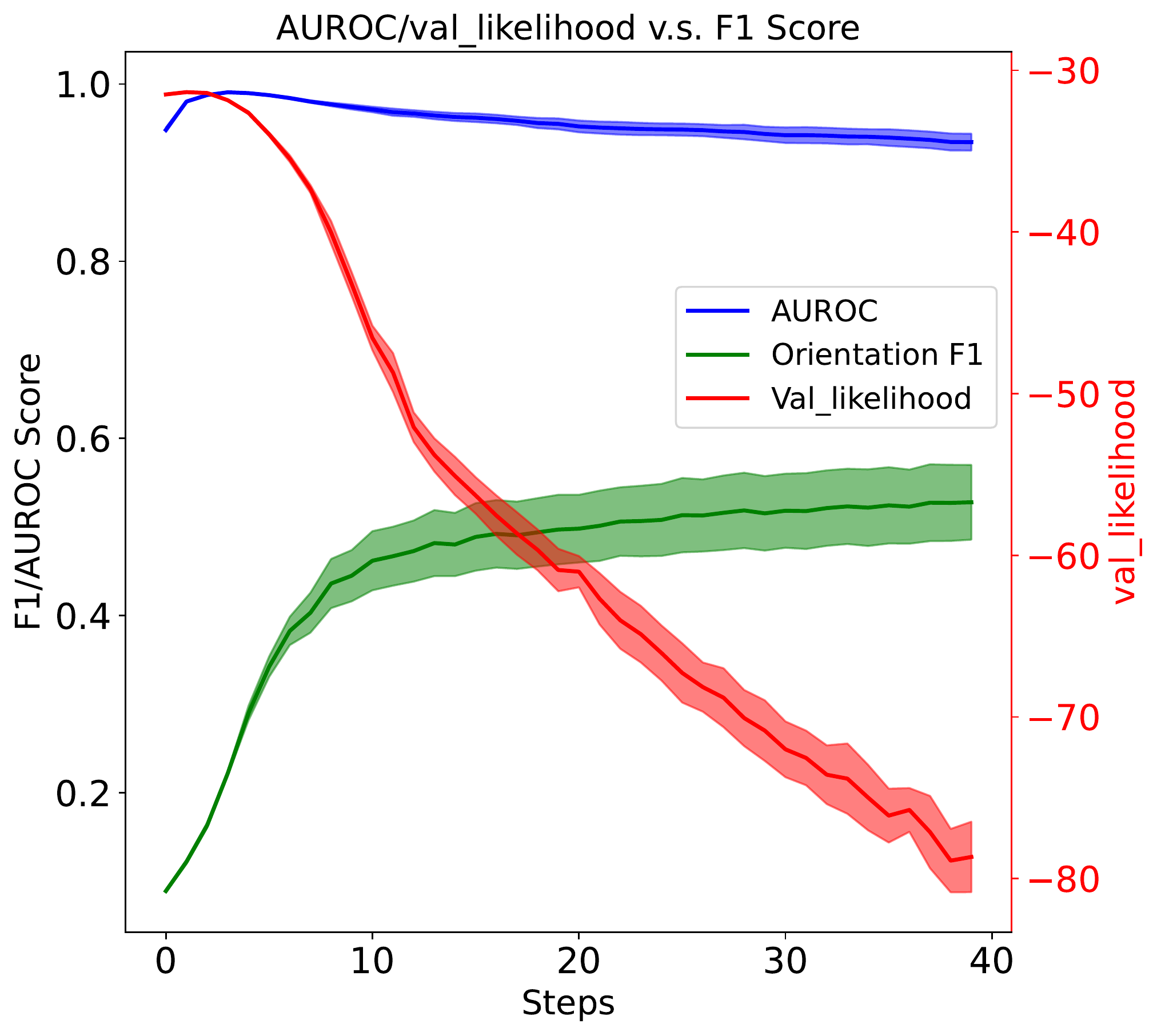}
    \caption{The curves of orientation F$_1$, AUROC and validation likelihood during training. Each curve is obtained by averaging over $5$ random seeds. The validation curve agrees well with the AUROC curve, but shows an opposite trends as F$_1$ curve. This potentially indicates model overfitting in the later stage of training.}
    \label{fig:Netsim AUROC F1 valid curve}
\end{figure}

\end{document}